\def\year{2022}\relax
\documentclass[letterpaper]{article} 
\usepackage{aaai22}  
\usepackage{times}  
\usepackage{helvet} 
\usepackage{courier}  
\usepackage[hyphens]{url}  
\usepackage{graphicx} 
\usepackage{natbib}  
\usepackage{caption} 
\DeclareCaptionStyle{ruled}{labelfont=normalfont,labelsep=colon,strut=off} 
\usepackage[utf8]{inputenc}
\usepackage[english]{babel}
\usepackage{xspace}
\usepackage{amssymb,amsmath,amsthm}
\usepackage{stmaryrd}
\usepackage{todonotes}
\usepackage{etoolbox,mathtools}
\usepackage{booktabs}
\usepackage{csquotes}
\usepackage{enumitem}
\usepackage{microtype}
\usepackage{xurl}
\newcommand{\ttodo}[4]{\ifthenelse{\equal{#1}{inline}}{\todo[inline, author=#2, color =
#3]{#4}}{\todo[color=#3]{#2: #4}}}

\newtheorem{theorem}{Theorem}

\newtheorem{definition}{Definition}

\newtheorem{corollary}{Corollary}

\newcommand{\wlg}{w.l.o.g.\ }

\newcommand{\wrt}{w.r.t.\ }
\newcommand{\ie}{i.e.\ }
\newcommand{\eg}{e.g.\ }

\def\define#1#2#3%
{%
\renewcommand*{\do}[1]{%
\expandafter\providecommand\csname
#1\endcsname{#2}
\expandafter\renewcommand\csname
#1\endcsname{#2}
}
\docsvlist{#3}
}

\define{#1}
{{\text{\upshape{\textsc{#1}}}}\xspace}
{NP,coNP,PSpace,ExpTime,ExpSpace,NLogTime,NExpTime,coNExpTime,AExpSpace,LogSpace}

\newcommand{\TwoExpTime}{\text{\upshape{\textsc{2ExpTime}}}\xspace}

\newcommand{\poly}{\text{\upshape{/poly}}\xspace}

\define{N#1}
{{\ensuremath{\mathbf{#1}}}\xspace}
{C,R,I}

\define{#1mc}
{{\ensuremath{\mathcal{#1}}}\xspace}
{A,B,C,D,E,F,G,H,I,J,K,L,M,N,O,P,Q,R,S,T,U,V,W,X,Y,Z}

\define{#1}
{{\ensuremath{\mathcal{#1}}}\xspace}
{EL,ELH,ELI,ELHI,ALC,ALCH,ALCI,ALCHI,ALCHQ,ALCHOIQ,S,SH,SHIQ,SRIQ,SHOIQ,SROIQ}

\define{#1}
{{\ensuremath{\mathsf{#1}}}\xspace}
{pre,eff,cond,add,del,ans,der,concept,anon,role,roles,n,edge,equal,reach,bad,inv,sup,ind,pos,U,max}

\define{#1}
{{\textsc{#1}}\xspace}
{Clipper}

\newcommand{\apply}[1]{\ensuremath{\llbracket #1 \rrbracket}\xspace}

\newcommand{\dneg}{Datalog$^\neg$\xspace}
\newcommand{\dsneg}{Datalog$^{S,\neg}$\xspace}
\newcommand{\adom}{\ensuremath{\mathcal{O}}\xspace}
\newcommand{\pto}{\phantom{{}\to{}}}
\newcommand{\tr}{\triangleright}

\newcommand{\ekabcompilation}{\textsf{Cal16}\xspace}
\newcommand{\krcompilation}{\textsf{Bor21}\xspace}
\newcommand{\datalogcompilation}{\textsf{Horn}\xspace}

\newcommand{\cats}{Cats\xspace}
\newcommand{\elevator}{Elevator\xspace}
\newcommand{\robot}{Robot\xspace}
\newcommand{\taskassignment}{TaskAssign\xspace}
\newcommand{\tpsa}{TPSA\xspace}
\newcommand{\vta}{VTA\xspace}
\newcommand{\vtaroles}{VTA-Roles\xspace}

\newcommand{\drones}{Drones\xspace}
\newcommand{\queens}{Queens\xspace}
\newcommand{\robotconj}{RobotConj\xspace}

\title{Expressivity of Planning with Horn Description Logic Ontologies \\ (Technical Report)}
\author{Stefan Borgwardt,\textsuperscript{\rm 1} Jörg Hoffmann,\textsuperscript{\rm 2} Alisa Kovtunova,\textsuperscript{\rm 1} Markus Krötzsch,\textsuperscript{\rm 1} Bernhard Nebel,\textsuperscript{\rm 3} Marcel Steinmetz\textsuperscript{\rm 2}}
\affiliations{
\textsuperscript{\rm 1}Faculty of Computer Science, Technische Universität Dresden, Germany,\\
\textsuperscript{\rm 2}Saarland University, Saarland Informatics Campus, Germany\\
\textsuperscript{\rm 3}Faculty of Engineering, University of Freiburg, Germany\\
firstname.lastname@tu-dresden.de,
lastname@cs.uni-saarland.de,
lastname@informatik.uni-freiburg.de}

\begin{document}

\maketitle


\begin{abstract}
State constraints in AI Planning globally restrict the legal
environment states. Standard planning languages make closed-domain and
closed-world assumptions. Here we address open-world state constraints
formalized by planning over a description logic (DL)
ontology. Previously, this combination of DL and planning has been
investigated for the light-weight DL DL-Lite. Here we propose a novel
compilation scheme into standard PDDL with derived predicates, which
applies to more expressive DLs and is based on the rewritability of DL
queries into Datalog with stratified negation. We also provide a new
rewritability result for the DL Horn-\ALCHOIQ, which allows us to
apply our compilation scheme to quite expressive ontologies. In
contrast, we show that in the slight extension Horn-\SROIQ no such
compilation is possible unless the weak exponential hierarchy
collapses. Finally, we show that our approach can outperform previous
work on existing benchmarks for planning with DL ontologies, and is
feasible on new benchmarks taking advantage of more expressive
ontologies.
\end{abstract}


\section{Introduction}
\label{sec:introduction}


AI planning is concerned with sequential decision making problems
where an agent needs to choose actions to achieve a goal, or to
maximize reward \cite{ghallab:etal:04}. Such
problems are compactly described in a declarative
language. Specifically, in the most basic (``classical'') version of
planning,
a planning task describes an initial state of the agent's environment,
a set of actions that can affect that environment, and a goal formula
that is to be satisfied.
In order to reach the goal, actions can be applied whenever their
preconditions are satisfied in the current state.
Here we are interested in \emph{state constraints}, constraints that
should hold globally, \ie at every state, in difference to
preconditions which merely need to hold locally.
%
%
%
Moreover, standard planning formalisms (based on variants of the PDDL
language \cite{pddl-handbook,haslum:etal:19}) follow
closed-domain and closed-world assumptions, in which absent facts are
assumed to be false and no new objects can be created. In particular,
these assumptions underly state constraints as can be specified in
PDDL3 \cite{gerevini:etal:ai-09}. Here we instead target open-domain,
open-world reasoning.

One way to do this is via
\emph{explicit-input Knowledge and Action Bases (eKABs)}
\cite{CMPS-IJCAI16}, where states (sets of ground atoms) are
interpreted using open-world semantics.
All states are subject to a \emph{background ontology}, which
describes high-level concepts and global state constraints.
Together, a state and an ontology describe a multitude of possible
worlds, which leaves room for unknown information about existing and unknown objects.
For example, an ontology could express that \enquote{everyone operating a machine works for an engineering department} and \enquote{everyone who works for a department is an employee} without explicitly identifying the department of each person or even stating that they are employees.
Action preconditions contain \emph{queries} that are evaluated under open-world semantics, \eg the query for all \enquote{employees} would return all machine operators among other people.
Finally, action effects add or remove atoms in the state, \eg reassign machines or departments.
This allows for a clean separation of what is directly observed (\ie the state, which contains, \eg operational data and sensor data) from what is indirectly inferred (using the ontology).

\citeauthor{CMPS-IJCAI16} \shortcite{CMPS-IJCAI16} investigated eKABs with ontologies and queries formulated in the description logic DL-Lite, which is a popular formalism for conceptual modeling \cite{CDL+-JAR07}.
Queries in this logic enjoy \emph{first-order rewritability}, which means that the queries and the ontology can be compiled into first-order (FO) formulas that are then evaluated under closed-world semantics.
Based on this property, the authors described a compilation of DL-Lite eKABs into classical PDDL planning tasks.
Later, this compilation was further optimized to enable practical planning with DL-Lite background ontologies \cite{BHKS-KR21}.

The goal of this paper is to extend the expressivity of state constraints in eKABs from the light-weight DL-Lite to more powerful description logics (DLs) \cite{BHLS-17}.
We mainly consider \emph{Horn} DLs, which are fragments of Horn-FOL \cite{KrRH-TOCL13,JPWZ-LICS19}.
We investigate for which Horn DLs compilations into PDDL exist.
For this purpose, we adapt the notion of \emph{compilation schemes} \cite{Nebe-JAIR00,ThHN-AI05}, which relate the expressivity of two formalisms.
On the one hand, polynomial compilation schemes show that the expressivity of DL eKABs is not higher than that of PDDL.
On the other hand, although such \mbox{eKABs} could be considered
syntactic sugar, they represent powerful tools that allow domain engineers to add open-world constraints to planning tasks. 

Our first contribution is a generic compilation scheme for any DL and queries that enjoy \emph{\dneg-rewritability}, which essentially allows us to compile them into a set of PDDL \emph{derived predicates}.
Using this, we can immediately employ many existing rewritability results from the DL literature for AI planning \cite{OrRS-KR10,EOS+-AAAI12,BiOr-RW15}.
We continue by describing a novel polynomial \dneg-rewriting for queries in the very expressive DL Horn-\ALCHOIQ \cite{OrRS-IJCAI11}, which allows us to extend the previous compilability result even further.
In contrast to this, we then show that such a compilation cannot exist (under a reasonable complexity-theoretic assumption) for the slightly more expressive DL Horn-\SROIQ \cite{OrRS-IJCAI11}.
For this, we follow the idea of a previous non-compilability result for PDDL with derived predicates into PDDL without derived predicates \cite{ThHN-AI05}.
This more or less draws a line between the standardized ontology languages OWL~1,\footnote{\url{http://www.w3.org/TR/owl-features/}} which results in planning tasks of equal expressivity as PDDL, and OWL~2,\footnote{\url{http://www.w3.org/TR/owl2-overview/}} where queries are strictly more expressive than PDDL.

While polynomial rewritings are nice in theory, they are often not practical due to a polynomial increase in the arity of predicates.
We therefore conclude the paper with an experimental evaluation that combines an existing practical implementation of an (exponential) \dneg-rewriting for Horn-\SHIQ \cite{EOS+-AAAI12} with our generic compilation scheme, compares this against previous approaches for DL-Lite eKABs on existing benchmarks \cite{CMPS-IJCAI16,BHKS-KR21}, and also introduces new benchmarks exploiting the newly increased expressivity.

\section{Preliminaries}
\label{sec:preliminiaries}

We first introduce description logics, \dneg, planning with derived predicates and ontologies, and compilations between these formalisms.
As usual, we use the symbol~$\models$ with two different meanings: open-world \emph{entailment} of formulas from sets of formulas, where all possible interpretations of arbitrary (even infinite) size are considered, and closed-world \emph{satisfaction} of a formula in a fixed, finite interpretation. It should always be clear from the context which one is used. 




\paragraph{Description Logics.}
Description logics are a family of KR formalisms \cite{BHLS-17} that describe open-world knowledge using \emph{axioms} in restricted first-order logic over unary and binary predicates.
Members of this family differ in their expressivity and complexity.
A \emph{TBox} (\emph{ontology}) is a finite set of \emph{DL axioms}, which can be seen as first-order sentences.
The precise syntax of these axioms depends on the specific DL that is used, but this is not important for most of the paper.
%
A \emph{state} (\emph{ABox}) is a finite set of \emph{ground atoms} $p(\vec{c})$, where $p$ is a predicate and $\vec{c}$ is a sequence of \emph{objects} (\emph{constants}). 
Since we use standard planning formalisms, we also allow predicates of arity higher than~$2$ in states, but those cannot occur in TBoxes, so essentially have a closed-world semantics.
For a state~$s$, $\adom(s)$ is the set of all objects occurring in~$s$.
Two special unary predicates are~$\top$ and~$\bot$, which always evaluate to \emph{true} and \emph{false}, respectively.

\paragraph{Queries.}
A \emph{conjunctive query (CQ)} is a formula of the form $q(\vec{x})=\exists\vec{y}.\phi(\vec{x},\vec{y})$, where $\phi$ is a conjunction of unary and binary atoms.
A \emph{union of CQs (UCQ)} is a disjunction of CQs.
An \emph{instance query (IQ)} is a CQ of the form~$p(x)$, where $p$ is unary.
The central reasoning problem is to decide whether $s,\Tmc,\theta\models q$, where $s$ is a state, \Tmc a TBox, and $\theta$ an assignment of objects from~$\adom(s)$ to the free variables in the (U)CQ~$q$.
In an abuse of notation, we may denote with~$\bot$ the CQ~$\exists x.\bot(x)$.
\emph{ECQs} were introduced to combine open-world and closed-world reasoning \cite{CDL+-IJCAI07}. We further extend ECQs by \enquote{closed-world atoms} that can also be of higher arity. ECQs are defined by the grammar
$Q ::= p(\vec{x}) \mid [q] \mid \lnot Q \mid Q\land Q \mid \exists y.Q$,
where~$p$ is a predicate, $\vec{x}$ are terms, $q$ is a UCQ, and $y$ is a variable.
%
%
%
The semantics of ECQs is defined as follows:
%
\[\begin{array}{@{}l@{\ \ }l@{\ \ }l@{}}
  s,\Tmc,\theta \models p(\vec{x}) & \text{iff} & s\models p(\theta(\vec{x})) \\
  s,\Tmc,\theta \models [q] & \text{iff} & s,\Tmc,\theta\models q \\
  s,\Tmc,\theta \models \lnot Q_1 & \text{iff} & s,\Tmc,\theta \not\models Q_1 \\
  s,\Tmc,\theta \models Q_1\land Q_2 & \text{iff} & s,\Tmc,\theta \models Q_1 \text{ and } s,\Tmc,\theta \models Q_2 \\
  s,\Tmc,\theta \models \exists y.Q_1 & \text{iff} & \exists o\in\adom(s) : s,\Tmc,\theta[y\mapsto o] \models Q_1
\end{array}\]

There is a difference between the ECQs $B(x)$, which is answered directly in the (closed-world) model described by~$s$, and $[B(x)]$, which is evaluated \wrt the TBox as well. For example, if $s=
\{C(a)\}$ and $\Tmc=\{C\sqsubseteq B\}$, then $B(x)$ is not satisfied by any instantiations, but $[B(x)]$ is.

\paragraph{\dneg.}
A \emph{\dneg rule} is a formula $p(\vec{x})\gets\Phi(\vec{x},\vec{y})$ whose \emph{body}~$\Phi(\vec{x},\vec{y})$ is a conjunction of literals and whose \emph{head}~$p(\vec{x})$ is an atom.
A set of \dneg rules~\Rmc is \emph{stratified} if the set of its predicates can be partitioned into $\Pmc_1,\dots,\Pmc_n$ such that, for all $p_i\in\Pmc_i$ and $p_i(\vec{x})\gets\Phi(\vec{x},\vec{y})\in\Rmc$,
\begin{itemize}
  \item if $p_j\in\Pmc_j$ occurs in $\Phi(\vec{x},\vec{y})$, then $j\le i$, and
  \item if $p_j\in\Pmc_j$ occurs negated in $\Phi(\vec{x},\vec{y})$, then $j<i$.
\end{itemize}
In the following, all sets of \dneg rules are stratified. 
Datalog is the restriction of \dneg to positive rule bodies.

All variables in \dneg rules are implicitly universally quantified.
Given a state~$s$ and a set of \dneg rules~\Rmc, we denote by $\Rmc(s)$ the minimal Herbrand
model of $s\cup\Rmc$. 

\begin{definition}
  \label{def:datalog-rewriting}
  A TBox~\Tmc and a UCQ~$q(\vec{x})$ are \emph{\dneg-rewritable} if there is a set of \dneg
  rules~$\Rmc_{\Tmc,q}$ with a predicate~$P_q$ such that, for all states~$s$ and
  substitutions~$\theta$ of $\vec{x}$ in $\adom(s)$, we have $s,\Tmc,\theta\models q(\vec{x})$
  iff $\Rmc_{\Tmc,q}(s)\models P_q(\theta(\vec{x}))$.
  
\end{definition}
A \dneg-rewriting may use additional predicates and constants, but only needs to be correct for the original symbols.
We talk about \emph{Datalog-rewritability} if the set $\Rmc_{\Tmc,q}$ does not contain negation.
A variety of such rewritability results for DLs exist, for example
a very complex (but polynomial-size) Datalog-rewriting for IQs over Horn-\ALCHOIQ \cite{OrRS-KR10},
a polynomial-size Datalog-rewriting for IQs in $\EL^{++}$ \cite{Kroe-IJCAI11}, or
an exponential-size Datalog-rewriting for UCQs over Horn-\SHIQ TBoxes implemented in the \Clipper system\footnote{\url{https://github.com/ghxiao/clipper}} \cite{EOS+-AAAI12},
which is polynomial for the sublogic $\ELH_\bot$ \cite{BiOr-RW15}.

\dneg-rewritability naturally extends to ECQs~$Q$: take the disjoint union~$\Rmc_{\Tmc,Q}$
of all~$\Rmc_{\Tmc,q}$ for UCQs~$q$ occurring in~$Q$ and construct an FO
formula~$Q_\Tmc$ by replacing each UCQ atom~$[q(\vec{x})]$ in~$Q$ with~$P_q(\vec{x})$.
Then $s,\Tmc,\theta\models Q(\vec{x})$ is equivalent to $\Rmc_{\Tmc,Q}(s)\models Q_\Tmc(\theta(\vec{x}))$ \cite{CDL+-IJCAI07}.


\paragraph{PDDL with Derived Predicates.}
We recall PDDL 2.1 extended with derived predicates
\cite{fox:long:jair-03,hoffmann:edelkamp:jair-05}.
In this context, states are viewed under the closed-world assumption and all sets are finite.

\begin{definition}
  A \emph{PDDL domain description} is a tuple $(\Pmc,\Pmc_\der,\Amc,\Rmc)$, where
    $\Pmc,\Pmc_\der$ are disjoint sets of predicates;
    \Amc~is a set of \emph{actions}; and
    \Rmc~is a set of \emph{rules}.
%
  An \emph{action} is of the form $(\vec{x},\pre,\eff)$, with \emph{parameters}~$\vec{x}$,
  \emph{precondition}~\pre, and a finite set \eff of \emph{effects}.
  The precondition is an FO formula over
  $\Pmc\cup\Pmc_\der$ with free variables from~$\vec{x}$, and an \emph{effect} is of the form $(\vec{y},\cond,\add,\del)$,
  where $\vec{y}$ are variables, \cond is an FO formula over $\Pmc\cup\Pmc_\der$ with free variables from
  $\vec{x}\cup\vec{y}$, \add is a finite set of atoms over~\Pmc (without $\Pmc_\der$) with free
  variables from $\vec{x}\cup\vec{y}$, and \del is a finite set of such negated
  atoms.
  \emph{Rules} are of the form $P(\vec{x})\leftarrow \phi(\vec{x})$ with $P\in\Pmc_\der$ and a first-order
  formula~$\phi$ over $\Pmc\cup\Pmc_\der$.
  The set~\Rmc must be \emph{stratified}, \ie fulfill the same condition as sets of \dneg rules when considering the rule bodies~$\phi(\vec{x})$ in NNF.

  A \emph{PDDL task} is a tuple $(\Delta,\Omc,I,G)$, where
    $\Delta$ is a PDDL domain description;
    \Omc is a finite set of \emph{objects} including the ones in~$\Delta$;
    $I$ is the \emph{initial state}; and
    $G$ is the \emph{goal}, a closed FO formula over $\Pmc\cup\Pmc_\der$ and constants from~\Omc.
\end{definition}

Derived predicates are not allowed to be modified by actions, \ie they are only determined by the
current state and the rules.
The semantics of rules is defined similarly to \dneg, \ie for a state~$s$, $\Rmc(s)$ is the
minimal Herbrand model obtained by exhaustively applying the rules in~$\Rmc$, stratum by stratum, to the facts in~$s$ in
order to populate the derived predicates~$\Pmc_\der$.
In fact, all such rule sets \Rmc can be reformulated into \dneg rule sets~\cite{AbHV-95,ThHN-AI05}.
Although the definition of derived predicates requires the head and body to have the same free variables, this is compatible with the semantics of \dneg since additional body variables can be viewed as implicitly existentially quantified.

For an action $a=(\vec{x},\pre,\eff)$ and~$\theta\colon\vec{x}\to\Omc$, 
the \emph{ground action} $\theta(a)$ 
has no parameters.
A ground action $a=(\pre,\eff)$ is \emph{applicable} in a state~$s$ if $\Rmc(s)\models\pre$ and its \emph{application} yields a new state
$s\llbracket a\rrbracket$ that contains a ground atom~$\alpha$ iff
  (1) there are $(\vec{y},\cond,\add,\del)\in\eff$ and~$\theta$ such
  that $\Rmc(s)\models\theta(\cond)$ and $\alpha\in\theta(\add)$; or
  (2) $\alpha\in s$ and for all $(\vec{y},\cond,\add,\del)\in\eff$ and~$\theta$ it holds that $\Rmc(s)\not\models\theta(\cond)$ or $\lnot\alpha\notin\theta(\del)$.
%
%
A \emph{plan}~$\pi$ is a sequence of ground actions such that $\pi$ is applicable in~$I$ and
$\Rmc(I\apply{\pi})\models G$.

\paragraph{eKABs.}
We recall \emph{explicit-input action and knowledge bases} \cite{CMPS-IJCAI16}, but slightly
adapt the notation to be consistent with PDDL notation.

\begin{definition}
  An \emph{eKAB domain description} is a tuple $(\Pmc,\Amc,\Tmc)$, where
    \Pmc is a finite set of predicates;
    \Amc is a finite set of \emph{DL actions}; and
    \Tmc is a TBox over the unary and binary predicates in~\Pmc.
%
  A \emph{DL action} is of the form $(\vec{x},\pre,\eff)$, 
  where
  $\pre$ is an ECQ over~\Pmc with free variables from~$\vec{x}$, and
  \eff consists of \emph{DL effect} of the form $(\vec{y},\cond,\add,\del)$, where \cond is an ECQ over~\Pmc with free variables from $\vec{x}\cup\vec{y}$, and \add and \del are as in PDDL.

  An \emph{eKAB (task)} is a tuple $(\Delta,\Omc,\Omc_0,I,G)$, where
    $\Delta$ is an eKAB domain description;
    \Omc is a possibly infinite set of \emph{objects};
    $\Omc_0$ is a finite subset of~\Omc including the objects from~$\Delta$;
    $I$ is the \emph{initial state} (over~$\Omc_0$), which is consistent with~\Tmc; and
    $G$ is the \emph{goal}, a closed ECQ over~\Pmc and constants from~$\Omc_0$.
\end{definition}

A ground action~$a$ is \emph{applicable} in~$s$ if $s\models\pre$ and $s\apply{a}\cup\Tmc$ is consistent.
The \emph{application} $s\apply{a}$ contains a fact~$\alpha$ iff
%
  (1') there are $(\vec{y},\cond,\add,\del)\in\eff$ and~$\theta\colon\vec{y}\to\adom(s)\cup\Omc_0$ such that $s,\Tmc,\theta\models\cond$ and
  $\alpha\in\theta(\add)$; or
  (2') $\alpha\in s$ and for all $(\vec{y},\cond,\add,\del)\in\eff$ and~$\theta$ as above it holds that
  $s,\Tmc,\theta\not\models\cond$ or $\lnot\alpha\notin\theta(\del)$.
%
%
A \emph{plan}~$\pi$ must be applicable in~$I$ and satisfy $I\apply{\pi},\Tmc\models G$.
Substitutions for effects range over $\adom(s)\cup\Omc_0$ since the TBox may contain objects from~$\Omc_0$.
In the following, we assume \wlg that $\Omc_0\subseteq\adom(s)$.

\paragraph{Additional Assumptions.}
Actions can refer to new objects (parameters~$\vec{x}$ that are not in the
precondition), and thereby increase the number of objects in a state.
To obtain manageable state transition systems, in the literature the assumption of \emph{state-boundedness} is often considered for eKAB-like formalisms; it requires that there exists a bound~$b$ such that any state reachable from~$I$ contains at most~$b$ objects~\cite{CDMP-RR13,DLPV-AAMAS14}.
For a fixed $b$, $b$-boundedness of an eKAB is decidable~\cite{DLPV-AAMAS14}.
Even if a given eKAB is not state-bounded, one could instead ask for the existence of a $b$-bounded plan \cite{ACOS-TOCL17}.
%
An abstraction result implies that any $b$-bounded eKAB can be reformulated into one where $|\Omc|=|\Omc_0|+n+b$, where $n$ is the maximum number of parameters of any action~\cite{CDMP-RR13,CMPS-IJCAI16}.
Any plan of the original eKAB can still be encoded using this finite set of objects.
Conversely, any abstract plan can be reformulated into a plan of the original eKAB by replacing the $n+b$ abstract objects by fresh objects from the original set~\Omc where necessary (\ie if those objects did not occur in the previous state).
We will also make this assumption here and assume for simplicity that $\Omc=\Omc_0$ is finite and denote both sets by~\Omc.

Even for a $b$-bounded eKAB, the TBox~\Tmc can entail the existence of objects that are not mentioned in a state~$s$.
These objects are not affected by the bound~$b$, because they are never explicitly materialized in~$s$.
Hence, the reasoning problems
still employ standard DL semantics rather than fixed-domain reasoning~\cite{GaRS-ECAI16}.

We also assume that goals of eKAB and PDDL tasks consist of a single (closed-world) atom $g(\vec{c})$. If that is not the case, we can introduce a new action with the goal formula $G(\vec{x})$ as precondition that adds $g(\vec{x})$ to the state. The parameters~$\vec{x}$ correspond to the constants~$\vec{c}$ in the original goal formula $G(\vec{c})$.
This assumption simplifies some of the formal definitions, but does not affect our main insights. Without this assumption, for example, the following definition of domain compilation~$f_\delta$ would also need to depend on the goal formula since we later need to compile all UCQs (also those in the goal) into derived predicates.

\paragraph{Compilations.}
To study the relative expressivity of these formalisms, we adapt the notion of compilation schemes \cite{Nebe-JAIR00,ThHN-AI05}.

\begin{definition}\label{def:poly-rewr}
  A \emph{compilation scheme} $\mathbf{f}$ from eKABs to PDDL is a tuple of functions $(f_\delta,f_o,f_i,f_g)$ that induces a function~$F$ from eKAB tasks $\Pi=(\Delta,\Omc,I,G)$ to PDDL tasks
  \[ F(\Pi) := \big(f_\delta(\Delta),\Omc\cup f_o(\Delta),f_i(\Omc,I),f_g(\Omc,G)\big) \]
  such that
    (A) there exists a plan for~$\Pi$ iff there exists a plan for~$F(\Pi$); and
    (B) $f_i$ and~$f_g$ are polynomial-time computable.
  If $\|f_\delta(\Delta)\|$ and~$\|f_o(\Delta)\|$ are bounded polynomially (exponentially) in~$\|\Delta\|$, then $\mathbf{f}$ is \emph{polynomial} (\emph{exponential}).
  
  If for every plan~$P$ solving an instance~$\Pi$ there exists a plan~$P'$ solving~$F(\Pi)$ such that $\|P'\|\le c\cdot\|P\|^n+k$ for positive integer constants~$c,n,k$, we say that $\mathbf{f}$ \emph{preserves plan size polynomially}. If $n=1$, it \emph{preserves plan size linearly}, and if additionally $c=1$, then it \emph{preserves plan size exactly}.
\end{definition}
%
To be considered of the same expressivity, there should be a polynomial compilation scheme between two formalisms that at least preserves plan size polynomially, but ideally exactly.
Compilation schemes for specific DLs are restricted to TBoxes~\Tmc formulated in the specified DL.
For example, there exists an exponential compilation scheme for DL-Lite eKABs that rewrites ECQs in-situ into FO conditions \cite{CMPS-IJCAI16}.
This compilation has been optimized by~\citeauthor{BHKS-KR21}
\shortcite{BHKS-KR21} by using derived predicates to simplify the conditions.
In contrast, the compilations we investigate in the following directly use \dneg-rewritings to compile UCQs into derived predicates.

\section{Compiling TBoxes into Derived Predicates}
\label{sec:generic-compilation}

We start by describing a generic compilation that exploits \dneg-rewritability of specific DLs to compile open-world ECQs into closed-world formulas using derived predicates.
%

One restriction of PDDL derived predicates is that they cannot occur in action effects.
However, \dneg-rewritings may derive new facts about the predicates occurring in the
state and TBox, which can also occur in action effects.
To circumvent this issue, we observe that query rewriting is only necessary for \emph{evaluating conditions},
but does not affect the states themselves.
Therefore, we separate the condition evaluation and action effects by using two disjoint signatures of predicates:
we use the \dneg-rewriting on a copy~$s'$ of the state~$s$ in which each original predicate~$P$
has been replaced by a copy~$P'$.
This copying process can be simulated by making~$P'$ a derived predicate with the rule
$P'(\vec{x})\leftarrow P(\vec{x})$.
In the following, we denote by $\Tmc'$ the result of replacing each predicate~$P$ in~\Tmc by~$P'$,
and likewise for ECQs~$Q$.

Another issue is that the rewriting may introduce additional constants, which are not allowed for instantiating actions, because that would change their behavior.
We simulate this via two new predicates $S$ and $N$ and a new action~$a_{S,N}$ with precondition $\lnot S$ and unconditional effect $S$ and $N(o)$ for all objects $o$ that are not in the original domain description.
All other action conditions are also extended by $S$ and $\lnot N(\vec{x}):=\bigwedge_{x\in\vec{x}}\lnot N(x)$ for their parameters~$\vec{x}$, to ensure that they can only be instantiated by the original objects.


\begin{definition}\label{def:compilation}
  Let $((\Pmc,\Amc,\Tmc),\Omc,I,G)$ be a $b$-bounded eKAB for which all UCQs
  as well as~$\bot$ are \dneg-rewritable \wrt \Tmc.
  Let $\Rmc_{\Tmc'}$ be the disjoint union of $\Rmc_{\Tmc',\bot}$ and all $\Rmc_{\Tmc',Q'}$ for
  ECQs~$Q$ in the eKAB.
  Then the PDDL task $((\Pmc\cup\{S,N\},\Pmc',\Amc',\Rmc'),\Omc',I,G')$ is obtained as follows:
  \begin{itemize}
    \item $\Pmc'$ consists of the predicates occurring in~$\Rmc_{\Tmc'}$;
    \item $\Amc'$ contains $a_{S,N}$ and all actions obtained from~\Amc by replacing preconditions~\pre by $S\land\lnot N(\vec{x})\land\lnot P_\bot\land\pre'_{\Tmc'}$ and effect conditions~\cond by~$\cond'_{\Tmc'}$;
    \item $\Rmc'=\Rmc_{\Tmc'}\cup\{P'(\vec{x})\leftarrow P(\vec{x})\mid P\in\Pmc\}$;
    \item $\Omc'=\Omc\cup\adom(\Rmc_{\Tmc'})$; and
    \item $G'=S\land\lnot P_\bot\land G$.
  \end{itemize}
\end{definition}

The goal does not need to be rewritten \wrt \Tmc since we assumed that it is a single (closed-world) atom.

\begin{theorem}
  \label{thm:generic-compilation}
  Def.~\ref{def:compilation} is a compilation scheme from \dneg-rewritable eKABs to PDDL that preserves plan size exactly.
\end{theorem}
\begin{proof}
  The new goal~$G'$ can be computed in polynomial time since we only add $S\land\lnot P_\bot$.
  Moreover, $a_{S,N}$ and the set of rules $\Rmc_{\Tmc'}$ depend only on the objects and conditions occurring in the original eKAB domain description.
  We show that all plans of either planning task are also plans for the other (modulo the initializing action~$a_{S,N}$).

  Consistency of $s\cup\Tmc$ is equivalent to $\Rmc_{\Tmc',\bot}(s')\not\models P_\bot$,
  where $s'$ is obtained from~$s$ by replacing each~$P$ with~$P'$.
  Hence, the rules $P'(\vec{x})\leftarrow P(\vec{x})$ and the conjuncts $\lnot P_\bot$ in all conditions ensure that all states reached while executing a plan for the PDDL task are
  consistent with~\Tmc.
  This includes the initial state since $I$ is consistent with~\Tmc by assumption.
  
  Similarly, for any ECQ~$Q$, we have $\Rmc'(s)\models Q'_{\Tmc'}(\theta(\vec{x}))$ iff
  $s',\Tmc',\theta\models Q'(\vec{x})$,
  %
  which is equivalent
  to $s,\Tmc,\theta\models Q(\vec{x})$.
  Due to $a_{S,N}$, the substitutions for instantiating action and effect conditions range over the same objects in both tasks.
  Hence, both formalisms allow equivalent action applications in each state and can reach a goal state by the same plans.
\end{proof}

For example, this immediately implies that eKABs with Horn-\SHIQ TBoxes have exponential compilations into PDDL with derived predicates (without negation) and for $\ELH_\bot$ and DL-Lite we even obtain polynomial compilations~\cite{EOS+-AAAI12,BiOr-RW15}.
The latter also holds for Horn-\SHOIQ if all conditions are restricted to EIQs~\cite{OrRS-KR10}.
In general, our construction applies to any ontology language where UCQs (or IQs) are \dneg-rewritable, and to any specific TBoxes and queries that happen to be \dneg-rewritable.


\section{A Polynomial Rewriting for Horn-\ALCHOIQ}
\label{sec:polynomial-compilation}

To extend the compilability results, we develop a polynomial-size rewriting for UCQs over Horn-\ALCHOIQ into \dneg.
It is based on a query answering approach developed by \citeauthor{CaDK-KR18} \shortcite{CaDK-KR18} and encodes the relevant definitions from their paper into \dsneg rules, which extend \dneg by \emph{set terms} that denote sets of objects. We then adapt a known polynomial translation to obtain a set of \dneg rules \cite{OrRS-KR10}.

The full details can be found in the appendix, 
but we describe the main ideas here. 
The rewriting starts by translating the Horn-\ALCHOIQ axioms of the given TBox~\Tmc into Datalog rules \cite{CaDK-KR18}.
However, since the resulting rule set is exponential, we here reformulate it into polynomially many \dsneg rules, loosely following ideas from \citeauthor{OrRS-KR10} \shortcite{OrRS-KR10}.
The original approach uses exponentially many constants~$t_X$, where $X$ is a set of unary predicates, to describe anonymous objects that satisfy~$X$.
In \dsneg, these individuals can directly be described by sets $\{X\}$ that can be used as arguments to predicates.
For example, our rewriting introduces a new predicate $\role(r,X,Y)$ to express that the objects represented by~$X$ and~$Y$ are connected by the binary predicate~$r$, which is now viewed as an additional element of~\Omc.
Using such additional predicates, the translation of the original Datalog rules by \citeauthor{CaDK-KR18} \shortcite{CaDK-KR18} is straightforward.

The remainder of the rewriting encodes the \emph{filtration phase} from that paper into several strata of \dsneg rules.
We use bespoke predicates to encode the constructions of expanded states and graphs in Definitions~7 and~8 in \cite{CaDK-KR18}, \eg to compute a partial expansion of the input state to obtain more query matches and an acyclicity check over a dependency graph between query variables to filter out spurious matches.
After a translation from \dsneg into \dneg \cite{OrRS-KR10}, correctness of the rewriting follows mostly from Theorem~3 in the paper by \citeauthor{CaDK-KR18} \shortcite{CaDK-KR18}.

\begin{theorem}
  \label{thm:polynomial-compilation}
  UCQs over Horn-\ALCHOIQ TBoxes are \dneg-rewritable with rewritings of polynomial size.
\end{theorem}
By Theorem~\ref{thm:generic-compilation}, we thus obtain a polynomial compilation scheme for Horn-\ALCHOIQ eKABs into PDDL that preserves plan size exactly.
Admittedly, this construction is rather complex, but theoretically very interesting, in particular in light of the next section.



\section{Non-Compilability for Expressive eKABs}
\label{sec:no-compilation}

As a counterpoint to the previous section, we now prove that polynomial compilations cannot exist for Horn-\SROIQ, not even if we allow the plan size to increase polynomially.
Horn-\SROIQ differs from Horn-\ALCHOIQ only in allowing one additional type of axiom, called \emph{complex role inclusions}.
The following result is inspired by a similar non-compilability result for PDDL with derived predicates \cite{ThHN-AI05}.
We start with some observations about the complexity of the involved problems.
The \emph{polynomial-step planning problem} is to decide whether a given planning task has a plan of length polynomial (for some given polynomial), and the \emph{1-step planning problem} is the special case where the polynomial is~$1$.

\begin{theorem}
  \label{thm:poly-step-pddl-exptime}
  The polynomial-step planning problem for PDDL is \ExpTime-complete.
\end{theorem}
\begin{proof}
  Hardness follows from the complexity of the 1-step planning problem for PDDL with derived predicates \cite[Theorem~1]{ThHN-AI05}.
  Membership can be seen as follows. In exponential time, we can enumerate all plans of polynomial length (for a fixed polynomial). For each such plan, we can check whether each ground action was applicable, which facts were generated or deleted, and whether the goal is satisfied in the end. The most complex part of this check is the evaluation of the derived predicates after each action, which can be done in exponential time~\cite{DEGV-ACMCS01}.
\end{proof}

\begin{theorem}
  \label{thm:one-step-horn-sroiq-twoexptime}
  The 1-step planning problem for Horn-\SROIQ eKABs is \TwoExpTime-complete.
\end{theorem}
\begin{proof}
  Hardness follows from the complexity of reasoning in Horn-\SROIQ~\cite{OrRS-KR10}.
  Membership holds since we can enumerate all candidate 1-step plans in \PSpace, the CQs in preconditions and the goal can be evaluated in \TwoExpTime \cite{OrRS-IJCAI11} and the remaining parts of the ECQs can be evaluated in \PSpace \cite{AbHV-95}.
\end{proof}

While these complexity results already indicate that reasoning in Horn-\SROIQ is more powerful than (polynomial) planning in PDDL with derived predicates, they tell us nothing about the relative expressivity of these two formalisms. There could still exist a polynomial-size compilation scheme from the former to the latter that preserves plan size polynomially, because the compilation can use arbitrary computational resources as long as the result is of polynomial size.

To prove that such a compilation indeed cannot exist, we follow \citeauthor{ThHN-AI05} \shortcite{ThHN-AI05} by using the notion of \emph{advice-taking} Turing machines \cite{KaLi-LEM82}. Such machines are equipped with an advice oracle~$a$, which is a function from positive integers to bit strings. On input~$w$, the machine receives the \emph{advice}~$a(\|w\|)$ and then starts its computation as usual. The advice depends only on the length of the input, but not on its contents. An advice oracle is \emph{polynomial} if the length of $a(\|w\|)$ is bounded polynomially in~$\|w\|$. $\ExpTime\poly$ (\emph{non-uniform} \ExpTime) is the class of problems that can be decided by Turing machines with polynomial advice and exponential time bound.

The following result shows that a polynomial compilation scheme from Horn-\SROIQ eKABs to PDDL would imply that the weak exponential hierarchy collapses completely.
The latter is considered to be unlikely; in particular, it would mean that one can eliminate any bounded quantifier prefix in second-order logic and Presburger arithmetic \cite{GoLV-MFCS95,Haas-LICS14}.

\begin{theorem}
  \label{thm:horn-sroiq-no-polynomial-compilation}
  Unless $\ExpTime^\NP=\ExpTime$, there is no polynomial compilation scheme from Horn-\SROIQ eKABs to PDDL preserving plan size polynomially.
\end{theorem}
\begin{proof}[Proof sketch]
  Let $M$ be a universal Turing machine with double-exponential time bound that can simulate all other such TMs. 
  %
  In the appendix, we show how to construct a family of Horn-\SROIQ eKAB domain descriptions $\Delta_n$ such that $M$ accepts a word~$w$ of length~$n$ iff $(\Delta_n,\Omc,I_w,g)$ has a plan of length~$1$.
  Here, $\Omc$ contains only the two objects~$o$ and~$e$, $I_w$ is a state that can be computed from~$w$ in polynomial time, and $g$ is a nullary predicate.
  This construction is based on the \TwoExpTime-hardness proof for Horn-\SROIQ \cite{OrRS-KR10}.

  Assume now that there is a compilation scheme~$\mathbf{f}=(f_\delta,f_o,f_i,f_g)$ from Horn-\SROIQ eKABs to PDDL preserving plan size polynomially.
  This scheme could be used as an advice oracle as follows.
  Let $M'$ be a TM with double-exponential time bound. Then $M'$ accepts $w'$ iff $M$ accepts $w=M'\#w'$ (in some fixed encoding).
  Let $n$ be the size of $w$.
  The compilation of~$\Delta_n$ to a PDDL domain description~$\Delta_n'=f_\delta(\Delta_n)$ as well as $f_o(\Delta_n)$ can be used as polynomial advice for a Turing machine that, on input $w$, computes $\Omc=\{o,e\}$, $I_w$, and $(\Delta_n',\Omc\cup f_o(\Delta_n),f_i(\Omc,I_w),f_g(\Omc,g))$, which can be done in polynomial time.
  It then decides polynomial-step plan existence for this PDDL task, which can be done in \ExpTime by Theorem~\ref{thm:poly-step-pddl-exptime} and is equivalent to deciding whether $M'$ accepts~$w'$ by Definition~\ref{def:compilation}.
  Overall, this implies that $\ExpTime^\NP\subseteq\TwoExpTime$ is included in $\ExpTime\poly$, and therefore $\ExpTime^\NP=\ExpTime$ \cite{BuHo-FSTTCS92}, which contradicts the assumption of the theorem.
\end{proof}

\begin{corollary}
  Unless $\ExpTime^\NP=\ExpTime$, in general there can be no polynomial \dneg-rewritings for IQs over Horn-\SROIQ TBoxes.
\end{corollary}
\begin{proof}
  By Theorem~\ref{thm:generic-compilation}, such a rewriting would yield a polynomial compilation from Horn-\SROIQ eKABs to PDDL preserving plan size exactly, which contradicts the assumption by Theorem~\ref{thm:horn-sroiq-no-polynomial-compilation}.
\end{proof}

We obtain a similar result also for the non-Horn DLs \SH and \ALCI, because for them similar \TwoExpTime-hardness proofs can be adapted \cite{Lutz-IJCAR08,ELOS-IJCAI09}.

\begin{theorem}
  \label{thm:sh-alci-no-polynomial-compilation}
  Unless $\ExpTime^\NP=\ExpTime$, there is no polynomial compilation scheme from \SH or \ALCI eKABs to PDDL preserving plan size polynomially.
\end{theorem}


\section{Experiments}
\label{sec:experiments}

While polynomial compilations are nice in theory, they have one major drawback:
the size of the rules and in particular the arity of the new predicates grows
polynomially with the input \cite{CaKr-IJCAI20}.
In contrast, the existing exponential compilation from Horn-\SHIQ to Datalog
uses rules of constant size and in many cases does not exhibit an exponential
blowup \cite{EOS+-AAAI12}.
Moreover, from a pragmatic perspective, it is the only Datalog rewriting for CQs
over Horn-DLs that has been implemented so far, in the \Clipper system.
We thus implemented our compilation from Section~\ref{sec:generic-compilation}
using \Clipper to answer the following
questions: 1)~Is the compilation feasible, \ie can the generated
classical planning tasks be handled by state-of-the-art planners? 2)~How does
our compilation perform against existing eKAB compilations?


For the experiments, we use the Fast Downward (FD) planning
system~\cite{helmert:jair-06} version 20.06 (the newest version as of August
2021), the main implementation platform for classical planning today.
We ran FD with a dual-queue greedy best-first search using the $h^{\textup{FF}}$
heuristic, a commonly used baseline in the planning literature.
%
All experiments were run on a computer with an Intel Core i5-4590 CPU@3.30GHz, and run time and memory cutoffs of 600s and 8GBs, respectively.
The benchmarks and the compiler are available online.\footnote{\label{ref:supp}\url{https://gitlab.perspicuous-computing.science/a.kovtunova/pddl-horndl}}

\paragraph{Implementation.} 
We encode Horn-\SHIQ eKAB tasks as ontology files accompanied by PDDL files whose syntax has been extended to allow conjunctive queries. The
ontology file uses Turtle syntax, which can be processed by any
off-the-shelf ontology tool.
Our compiler reads the CQ-PDDL and ontology files, and generates a classical
planning task in standard PDDL format.
The Datalog rewriting is generated with \Clipper \cite{EOS+-AAAI12}.
The compilation additionally normalizes complex conditions via
a Tseitin-like transformation~\cite{T-CL83}, which has been shown to be effective before \cite{BHKS-KR21}.

\begin{figure}
  \centering
  \includegraphics[width=0.47\textwidth]{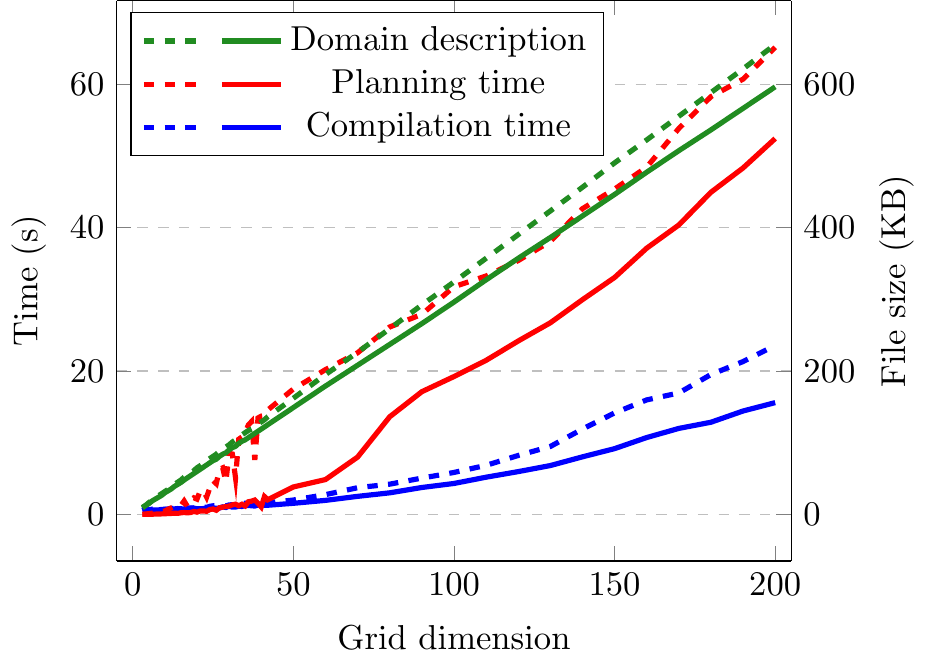}
  \caption{Comparison of the \robot (dashed) and \robotconj (solid) domains \wrt domain description size, compilation and planning times.}
  \label{fig:results}
\end{figure}

\begin{table*}
  \centering
  {
  \footnotesize
\begin{tabular}{|lr|rrr|rrr|rrr|rrr|}
\hline
& & \multicolumn{3}{c|}{\# solved} & \multicolumn{3}{c|}{\# compiled} &  \multicolumn{3}{c|}{planning time} & \multicolumn{3}{c|}{compilation time} \\
Domain & \# & \ekabcompilation & \krcompilation & \datalogcompilation & \ekabcompilation & \krcompilation & \datalogcompilation & \ekabcompilation & \krcompilation & \datalogcompilation & \ekabcompilation & \krcompilation & \datalogcompilation \\
\hline
\hline
\cats & 20 & 14 & \textbf{20} & \textbf{20} & 20 & 20 & 20 & 63.46 & 0.13 & \textbf{0.03} & \textbf{0.16} & 0.22 & 0.65 \\
\elevator & 20 & 20 & 20 & 20 & 20 & 20 & 20 & 0.36 & 0.30 & \textbf{0.03} & \textbf{0.60} & 0.73 & 0.66 \\
\robot & 20 & 4 & 12 & \textbf{20} & 12 & 12 & \textbf{20} & 15.05 & 10.10 & \textbf{0.11} & 138.52 & 138.55 & \textbf{0.75} \\
\taskassignment & 20 & 3 & \textbf{20} & \textbf{20} & 20 & 20 & 20 & 0.81 & 0.12 & \textbf{0.06} & 2.87 & 39.21 & \textbf{0.66} \\
\tpsa & 15 & 14 & 5 & \textbf{15} & \textbf{15} & 5 & \textbf{15} & 2.01 & 2.42 & \textbf{0.30} & 0.84 & 25.37 & \textbf{0.59} \\
\vta & 15 & \textbf{15} & 13 & \textbf{15} & 15 & 15 & 15 & 23.06 & 371.56 & \textbf{16.91} & \textbf{0.33} & 1.21 & 0.65 \\
\vtaroles & 15 & \textbf{15} & 5 & \textbf{15} & \textbf{15} & 5 & \textbf{15} & 2.25 & 11.61 & \textbf{1.36} & \textbf{0.59} & 95.53 & 0.66 \\
\hline
$\sum$ & 125 & 85 & 95 & \textbf{125} & 117 & 97 & \textbf{125} & 19.99 & 77.33 & \textbf{3.59} & 18.01 & 31.84 & \textbf{0.66} \\
\hline
\hline
\drones & 24 &  &  & {20} &  &  & {24} &  &  & {101.42} &  &  & {0.69} \\
\queens & 30 &  &  & {15} &  &  & {30} &  &  & {21.66} &  &  & {0.69} \\
\robotconj & 56 &  &  & {56} &  &  & {56} &  &  & {8.14} &  &  & {2.77} \\
\hline
$\sum$ & 110 &  &  & {91} &  &  & {110} &  &  & {30.87} &  &  & {1.75} \\
\hline
\end{tabular}

  }
  \caption{Per-domain aggregated statistics: ``\# solved'' number of instances
  solved by the planner;
  ``\# compiled'' number of instances for which the 
  compilers could generate the PDDL input files for the planner;
  ``planning time'' average planner run
  time over the commonly solved instances;
  ``compilation time'' average time of generating the PDDL
  files.}\label{tab:all}
\end{table*}

\paragraph{Benchmarks.}

Our benchmark collection consists of $125$ instances adapted from existing
DL-Lite eKAB benchmarks \cite{CMPS-IJCAI16,BHKS-KR21}, and $110$ newly created instances.
A detailed description can be found in the appendix.
We manually translated the existing eKAB domains (\cats, \elevator, \robot, \taskassignment, \tpsa, \vta, and \vtaroles) into the format described above. Modifications almost exclusively pertained to extracting the ontology
from the eKAB description into a separate Turtle file and moving so-called condition-action rules \cite{CMPS-IJCAI16} into action preconditions. The translated instances are equivalent to the originals.

Since Horn-\SHIQ is more expressive than DL-Lite, we also created $3$ new domains (\drones, \queens, and \robotconj)
in which we make use of conjunctions, qualified existential restrictions occurring with negative polarity, and symmetric and transitive relations, all of which are not supported by DL-Lite \cite{BHLS-17}.
\drones describes a complex 2D drone navigation problem, in which drones need to be
moved to avoid critical situations; the latter are described in the
ontology using axioms with qualified existential
restrictions and symmetric relations.
\queens generalizes the eight queens
puzzle to board sizes $n\in\{5,\dots, 10\}$ and numbers of queens
$m\in\{n-4,\dots, n\}$. Queens are initially placed randomly on the board and need to be moved to a configuration where no queen threatens another. The ontology contains a symmetric, transitive relation to
describe legal moves.
\robotconj is a redesign of \robot that moves some of the complexity from actions into the ontology. The original
\robot benchmark encodes static knowledge about 2D grid cell adjacency in the
action descriptions, which can be encoded much more
naturally in the ontology using conjunctions.
Note that the original \robot benchmark consists of 20 instances (grid sizes $3\times 3$ up to $22\times 22$), whereas for the new \robotconj we included 56 instances (up to $200\times 200$) since they could be easily handled by our compiler.

\paragraph{Scalability Study.}

We use \robot and \robotconj to analyze how our compilation performs as
a function of domain description size (including the ontology).
Figure~\ref{fig:results} depicts the results
for $56$ instances of each domain, obtained by scaling the grid from $3\times 3$
to $200 \times 200$.
In both domains, the file size is directly
proportional to size of the grid.
Even the largest tested
instance could be compiled and solved in less than~$90$ seconds, attesting the
feasibility of our approach.
The increased complexity of \robotconj's ontology does not
affect the performance. On the contrary, both compilation and planning for \robotconj are actually consistently faster than for \robot, due to the simplified actions.


\paragraph{Comparison to DL-Lite Compilations.}

We compare to \ekabcompilation, the original DL-Lite eKAB compiler
\cite{CMPS-IJCAI16}, and to \krcompilation, its recently introduced optimization using derived predicates to compile away complex formulas
\cite{BHKS-KR21}.  We refer to our compilation by
\datalogcompilation. Table~\ref{tab:all} gives a summary of the results. 
\ekabcompilation and \krcompilation were only run on the original DL-Lite eKAB
benchmarks. 


Considering the DL-Lite benchmark part, {\datalogcompilation} has similar or
better performance than {\ekabcompilation} and {\krcompilation}. In \robot, the
\ekabcompilation compilation (and hence also \krcompilation) could only process the
$12$ smallest instances (up to grid size $14\times 14$), exceeding the $600$ seconds time
limit thereafter.
While \ekabcompilation and \krcompilation both show a blow-up
in compilation time in some domains, our new \datalogcompilation compiler could
process all instances in less than $1$ second on average. The compilation time
of \datalogcompilation is almost consistently larger than $0.6$ seconds, which
can be attributed to the fixed overhead of calling \Clipper.
While the compilation time of \datalogcompilation is very competitive with the previous DL-Lite compilers, the planner's performance statistics
really substantiate this advantage.
Regarding planning time and the number of solved instances, {\datalogcompilation} significantly outperforms both alternatives on the DL-Lite benchmarks. 

The more complex ontologies in the Horn-\SHIQ benchmark part did not pose
a challenge to \datalogcompilation. All instances could still be translated
within $3$ seconds on average. The constructed instances of \drones and \queens are
however much more challenging from a planning perspective. Contrary to the
DL-Lite benchmarks before, average planning runtime is higher, and some instances
could not be solved in time.  The difficulty of the instances was chosen
intentionally, with the purpose of creating challenging problems for future
work.



\section{Conclusion}
\label{sec:conclusion}

We have shown that adding Horn-DL background ontologies often does not increase the expressivity of PDDL planning tasks.
This is due to \dneg-rewritability, which allows us to reduce open-world to closed-world reasoning.
However, adding more axiom types (Horn-\SROIQ) or using non-Horn DLs (\SH or \ALCI) increases the expressivity beyond PDDL, unless the weak exponential hierarchy collapses.
An evaluation of our generic compilation approach using the \Clipper system demonstrates the feasibility of using \dneg-rewritings, even compared to more specialized compilation schemes for the smaller logic DL-Lite.
Moreover, we have contributed additional benchmarks to showcase the increased expressivity of our proposed approach.

In future work, we will investigate the existence of a polynomial compilation scheme for Horn-\SHOIQ, whose expressivity lies between that of Horn-\ALCHOIQ and Horn-\SROIQ.
We also want to investigate planning formalisms with different effect semantics, \eg the one described by \citeauthor{DORS-JAIR21} \shortcite{DORS-JAIR21}.





\section*{Acknowledgements}
This work is supported by DFG grant 389792660 as part of TRR~248 -- CPEC 
(\url{https://perspicuous-computing.science}).

\bibliography{bibliography}

\begin{thebibliography}{40}
\providecommand{\natexlab}[1]{#1}

\bibitem[{Abiteboul, Hull, and Vianu(1995)}]{AbHV-95}
Abiteboul, S.; Hull, R.; and Vianu, V. 1995.
\newblock \emph{Foundations of Databases}.
\newblock Addison-Wesley.

\bibitem[{Ahmetaj et~al.(2017)Ahmetaj, Calvanese, Ortiz, and
  Šimkus}]{ACOS-TOCL17}
Ahmetaj, S.; Calvanese, D.; Ortiz, M.; and Šimkus, M. 2017.
\newblock Managing Change in Graph-Structured Data Using Description Logics.
\newblock \emph{ACM Transactions on Computational Logic}, 18(4): 27:1--27:35.

\bibitem[{Baader et~al.(2017)Baader, Horrocks, Lutz, and Sattler}]{BHLS-17}
Baader, F.; Horrocks, I.; Lutz, C.; and Sattler, U. 2017.
\newblock \emph{An Introduction to Description Logic}.
\newblock Cambridge University Press.

\bibitem[{Bienvenu and Ortiz(2015)}]{BiOr-RW15}
Bienvenu, M.; and Ortiz, M. 2015.
\newblock Ontology-Mediated Query Answering with Data-Tractable Description
  Logics.
\newblock In Faber, W.; and Paschke, A., eds., \emph{Reasoning Web. 11th Int.\
  Summer School}, volume 9203 of \emph{Lecture Notes in Computer Science},
  218--307. Springer.

\bibitem[{Borgwardt et~al.(2021)Borgwardt, Hoffmann, Kovtunova, and
  Steinmetz}]{BHKS-KR21}
Borgwardt, S.; Hoffmann, J.; Kovtunova, A.; and Steinmetz, M. 2021.
\newblock Making {DL-Lite} Planning Practical.
\newblock In Bienvenu, M.; and Lakemeyer, G., eds., \emph{Proc.\ of the 18th
  Int.\ Conf.\ on Principles of Knowledge Representation and Reasoning
  (KR'21)}, 641--645. IJCAI.

\bibitem[{Buhrman and Homer(1992)}]{BuHo-FSTTCS92}
Buhrman, H.; and Homer, S. 1992.
\newblock Superpolynomial Circuits, Almost Sparse Oracles and the Exponential
  Hierarchy.
\newblock In Shyamasundar, R., ed., \emph{Proc.\ of the 12th Conf.\ on
  Foundations of Software Technology and Theoretical Computer Science
  (FSTTCS'92)}, volume 652 of \emph{Lecture Notes in Computer Science},
  116--127. Springer-Verlag.

\bibitem[{Calvanese et~al.(2007{\natexlab{a}})Calvanese, De~Giacomo, Lembo,
  Lenzerini, and Rosati}]{CDL+-IJCAI07}
Calvanese, D.; De~Giacomo, G.; Lembo, D.; Lenzerini, M.; and Rosati, R.
  2007{\natexlab{a}}.
\newblock {EQL-Lite}: {E}ffective First-Order Query Processing in Description
  Logics.
\newblock In Veloso, M.~M., ed., \emph{20th Int.\ Joint Conf.\ on Artificial
  Intelligence (IJCAI)}, 274--279.

\bibitem[{Calvanese et~al.(2007{\natexlab{b}})Calvanese, {De Giacomo}, Lembo,
  Lenzerini, and Rosati}]{CDL+-JAR07}
Calvanese, D.; {De Giacomo}, G.; Lembo, D.; Lenzerini, M.; and Rosati, R.
  2007{\natexlab{b}}.
\newblock Tractable Reasoning and Efficient Query Answering in Description
  Logics: {T}he {\textit{DL-Lite}} Family.
\newblock \emph{Journal of Automated Reasoning}, 39(3): 385--429.

\bibitem[{Calvanese et~al.(2013)Calvanese, {De Giacomo}, Montali, and
  Patrizi}]{CDMP-RR13}
Calvanese, D.; {De Giacomo}, G.; Montali, M.; and Patrizi, F. 2013.
\newblock Verification and Synthesis in Description Logic Based Dynamic
  Systems.
\newblock In Faber, W.; and Lembo, D., eds., \emph{7th Int.\ Conf.\ Web
  Reasoning and Rule Systems (RR)}, 50--64. Springer.

\bibitem[{Calvanese et~al.(2016)Calvanese, Montali, Patrizi, and
  Stawowy}]{CMPS-IJCAI16}
Calvanese, D.; Montali, M.; Patrizi, F.; and Stawowy, M. 2016.
\newblock Plan Synthesis for Knowledge and Action Bases.
\newblock In Kambhampati, S., ed., \emph{25th Int.\ Joint Conf.\ on Artificial
  Intelligence (IJCAI)}, 1022--1029. AAAI Press.

\bibitem[{Carral, Dragoste, and Krötzsch(2018)}]{CaDK-KR18}
Carral, D.; Dragoste, I.; and Krötzsch, M. 2018.
\newblock The Combined Approach to Query Answering in
  Horn-{$\mathcal{ALCHOIQ}$}.
\newblock In Thielscher, M.; Toni, F.; and Wolter, F., eds., \emph{Proc.\ of
  the 16th Int.\ Conf.\ on Principles of Knowledge Representation and Reasoning
  (KR'18)}, 339--348. AAAI Press.

\bibitem[{Carral and Kr{\"o}tzsch(2020)}]{CaKr-IJCAI20}
Carral, D.; and Kr{\"o}tzsch, M. 2020.
\newblock Rewriting the Description Logic {$\mathcal{ALCHIQ}$} to Disjunctive
  Existential Rules.
\newblock In Bessiere, C., ed., \emph{Proc.\ of the 29th Int.\ Joint Conf.\ on
  Artificial Intelligence and the 17th Pacific Rim Int.\ Conf.\ on Artificial
  Intelligence (IJCAI-PRICAI'20)}, 1777--1783. IJCAI.

\bibitem[{Dantsin et~al.(2001)Dantsin, Eiter, Gottlob, and
  Voronkov}]{DEGV-ACMCS01}
Dantsin, E.; Eiter, T.; Gottlob, G.; and Voronkov, A. 2001.
\newblock Complexity and Expressive Power of Logic Programming.
\newblock \emph{ACM Computing Surveys}, 33(3): 374--425.

\bibitem[{{De Giacomo} et~al.(2014){De Giacomo}, Lespérance, Patrizi, and
  Vassos}]{DLPV-AAMAS14}
{De Giacomo}, G.; Lespérance, Y.; Patrizi, F.; and Vassos, S. 2014.
\newblock Progression and Verification of Situation Calculus Agents with
  Bounded Beliefs.
\newblock In amd Paul~Scerri, A.~L.; Bazzan, A.; and Huhns, M., eds.,
  \emph{13th Int.\ Conf.\ on Autonomous Agents and Multiagent Systems (AAMAS)},
  141--148. ACM.

\bibitem[{De~Giacomo et~al.(2021)De~Giacomo, Oriol, Rosati, and
  Savo}]{DORS-JAIR21}
De~Giacomo, G.; Oriol, X.; Rosati, R.; and Savo, D.~F. 2021.
\newblock Instance-Level Update in {DL-Lite} Ontologies through First-Order
  Rewriting.
\newblock \emph{Journal of Artificial Intelligence Research}, 70: 1335--1371.

\bibitem[{Eiter et~al.(2009{\natexlab{a}})Eiter, Lutz, Ortiz, and
  {\v{S}}imkus}]{ELOS-IJCAI09}
Eiter, T.; Lutz, C.; Ortiz, M.; and {\v{S}}imkus, M. 2009{\natexlab{a}}.
\newblock Query Answering in Description Logics with Transitive Roles.
\newblock In Boutilier, C., ed., \emph{Proc.\ of the 21st Int.\ Joint Conf.\ on
  Artificial Intelligence (IJCAI'09)}, 759--764. AAAI Press.

\bibitem[{Eiter et~al.(2009{\natexlab{b}})Eiter, Lutz, Ortiz, and
  {\v{S}}imkus}]{ELOS-INFSYS09}
Eiter, T.; Lutz, C.; Ortiz, M.; and {\v{S}}imkus, M. 2009{\natexlab{b}}.
\newblock Query Answering in Description Logics with Transitive Roles.
\newblock INFSYS Research Report 1843-09-02, Institut f{\"u}r
  Informationssysteme, TU Wien.

\bibitem[{Eiter et~al.(2012)Eiter, Ortiz, {\v{S}}imkus, Tran, and
  Xiao}]{EOS+-AAAI12}
Eiter, T.; Ortiz, M.; {\v{S}}imkus, M.; Tran, T.-K.; and Xiao, G. 2012.
\newblock Query Rewriting for {H}orn-$\mathcal{SHIQ}$ Plus Rules.
\newblock In Hoffmann, J.; and Selman, B., eds., \emph{Proc.\ of the 26th AAAI
  Conf.\ on Artificial Intelligence (AAAI'12)}, 726--733. AAAI Press.

\bibitem[{Fox and Long(2003)}]{fox:long:jair-03}
Fox, M.; and Long, D. 2003.
\newblock {PDDL2.1:} An Extension to {PDDL} for Expressing Temporal Planning
  Domains.
\newblock \emph{Journal of Artificial Intelligence Research}, 20: 61--124.

\bibitem[{Gaggl, Rudolph, and Schweizer(2016)}]{GaRS-ECAI16}
Gaggl, S.~A.; Rudolph, S.; and Schweizer, L. 2016.
\newblock Fixed-Domain Reasoning for Description Logics.
\newblock In Kaminka, G.~A.; Fox, M.; Bouquet, P.; H{\"{u}}llermeier, E.;
  Dignum, V.; Dignum, F.; and van Harmelen, F., eds., \emph{Proc.\ of the 22nd
  Eur.\ Conf.\ on Artificial Intelligence (ECAI'16)}, volume 285 of
  \emph{Frontiers in Artificial Intelligence and Applications}, 819--827. {IOS}
  Press.

\bibitem[{Gerevini et~al.(2009)Gerevini, Haslum, Long, Saetti, and
  Dimopoulos}]{gerevini:etal:ai-09}
Gerevini, A.; Haslum, P.; Long, D.; Saetti, A.; and Dimopoulos, Y. 2009.
\newblock Deterministic planning in the fifth international planning
  competition: {PDDL3} and experimental evaluation of the planners.
\newblock \emph{Artificial Intelligence}, 173(5-6): 619--668.

\bibitem[{Ghallab, Nau, and Traverso(2004)}]{ghallab:etal:04}
Ghallab, M.; Nau, D.; and Traverso, P. 2004.
\newblock \emph{Automated Planning: {Theory} and Practice}.
\newblock Morgan Kaufmann.

\bibitem[{Gottlob, Leone, and Veith(1995)}]{GoLV-MFCS95}
Gottlob, G.; Leone, N.; and Veith, H. 1995.
\newblock Second Order Logic and the Weak Exponential Hierarchies.
\newblock In Wiedermann, J.; and H{\'a}jek, P., eds., \emph{Proc.\ of the 20th
  Int.\ Symp.\ on Mathematical Foundations of Computer Science (MFCS'95)},
  volume 969 of \emph{Lecture Notes in Computer Science}, 66--81.
  Springer-Verlag.

\bibitem[{Haase(2014)}]{Haas-LICS14}
Haase, C. 2014.
\newblock Subclasses of Presburger Arithmetic and the Weak {EXP} Hierarchy.
\newblock In Henzinger, T.; and Miller, D., eds., \emph{Proc.\ of the Joint
  Meeting of the 23rd EACSL Annual Conf.\ on Computer Science Logic (CSL) and
  the 29th Annual ACM/IEEE Symp.\ on Logic in Computer Science (LICS)},
  14:1--14:10. ACM.

\bibitem[{Haslum et~al.(2019)Haslum, Lipovetzky, Magazzeni, and
  Muise}]{haslum:etal:19}
Haslum, P.; Lipovetzky, N.; Magazzeni, D.; and Muise, C. 2019.
\newblock An introduction to the planning domain definition language.
\newblock \emph{Synthesis Lectures on Artificial Intelligence and Machine
  Learning}, 13(2): 1--187.

\bibitem[{Helmert(2006)}]{helmert:jair-06}
Helmert, M. 2006.
\newblock The Fast Downward Planning System.
\newblock \emph{Journal of Artificial Intelligence Research}, 26: 191--246.

\bibitem[{Hoffmann and Edelkamp(2005)}]{hoffmann:edelkamp:jair-05}
Hoffmann, J.; and Edelkamp, S. 2005.
\newblock The Deterministic Part of {IPC-4}: {A}n Overview.
\newblock \emph{Journal of Artificial Intelligence Research}, 24: 519--579.

\bibitem[{Hoffmann et~al.(2008)Hoffmann, Weber, Scicluna, Kacmarek, and
  Ankolekar}]{hoffmann:etal:icwe-08}
Hoffmann, J.; Weber, I.; Scicluna, J.; Kacmarek, T.; and Ankolekar, A. 2008.
\newblock Combining Scalability and Expressivity in the Automatic Composition
  of Semantic Web Services.
\newblock In \emph{8th International Conference on Web Engineering (ICWE'08)}.

\bibitem[{Jung et~al.(2019)Jung, Papacchini, Wolter, and
  Zakharyaschev}]{JPWZ-LICS19}
Jung, J.~C.; Papacchini, F.; Wolter, F.; and Zakharyaschev, M. 2019.
\newblock Model Comparison Games for Horn Description Logics.
\newblock In Bouyer, P., ed., \emph{Proc.\ of the 34th Annual IEEE Symp.\ on
  Logic in Computer Science (LICS'19)}, 1--14. IEEE.

\bibitem[{Karp and Lipton(1982)}]{KaLi-LEM82}
Karp, R.~M.; and Lipton, R.~J. 1982.
\newblock Turing Machines that Take Advice.
\newblock \emph{L'Enseignement Math{\'e}matique}, 28: 191--209.

\bibitem[{Kr{\"o}tzsch(2011)}]{Kroe-IJCAI11}
Kr{\"o}tzsch, M. 2011.
\newblock Efficient Rule-Based Inferencing for {OWL EL}.
\newblock In Walsh, T., ed., \emph{Proc.\ of the 22nd Int.\ Joint Conf.\ on
  Artificial Intelligence (IJCAI'11)}, 2668--2773. AAAI Press.

\bibitem[{Kr{\"o}tzsch, Rudolph, and Hitzler(2013)}]{KrRH-TOCL13}
Kr{\"o}tzsch, M.; Rudolph, S.; and Hitzler, P. 2013.
\newblock Complexities of Horn Description Logics.
\newblock \emph{{ACM} Transactions on Computational Logic}, 14(1): 2:1--2:36.

\bibitem[{Lutz(2007)}]{Lutz-DL07}
Lutz, C. 2007.
\newblock Inverse Roles Make Conjunctive Queries Hard.
\newblock In Calvanese, D.; Franconi, E.; Haarslev, V.; Lembo, D.; Motik, B.;
  Turhan, A.-Y.; and Tessaris, S., eds., \emph{Proc.\ of the 20th Int.\
  Workshop on Description Logics (DL'07)}, volume 250 of \emph{CEUR Workshop
  Proceedings}, 100--111.

\bibitem[{Lutz(2008)}]{Lutz-IJCAR08}
Lutz, C. 2008.
\newblock The Complexity of Conjunctive Query Answering in Expressive
  Description Logics.
\newblock In \emph{Proc.\ of the 4th Int.\ Joint Conf.\ on Automated Reasoning
  (IJCAR'08)}, volume 5195 of \emph{Lecture Notes in Artificial Intelligence},
  179--193. Springer-Verlag.

\bibitem[{McDermott et~al.(1998)McDermott, Ghallab, Howe, Knoblock, Ram,
  Veloso, Weld, and Wilkins}]{pddl-handbook}
McDermott, D.; Ghallab, M.; Howe, A.; Knoblock, C.; Ram, A.; Veloso, M.; Weld,
  D.; and Wilkins, D. 1998.
\newblock \emph{The {PDDL} Planning Domain Definition Language}.
\newblock The {AIPS-98} Planning Competition Comitee.

\bibitem[{Nebel(2000)}]{Nebe-JAIR00}
Nebel, B. 2000.
\newblock On the Compilability and Expressive Power of Propositional Planning
  Formalisms.
\newblock \emph{Journal of Artificial Intelligence Research}, 12: 271--315.

\bibitem[{Ortiz, Rudolph, and {\v{S}}imkus(2010)}]{OrRS-KR10}
Ortiz, M.; Rudolph, S.; and {\v{S}}imkus, M. 2010.
\newblock Worst-case Optimal Reasoning for the {H}orn-{DL} Fragments of {OWL} 1
  and 2.
\newblock In Lin, F.; Sattler, U.; and Truszczynski, M., eds., \emph{Proc.\ of
  the 12th Int.\ Conf.\ on Principles of Knowledge Representation and Reasoning
  (KR'10)}, 269--279. AAAI Press.

\bibitem[{Ortiz, Rudolph, and {\v{S}}imkus(2011)}]{OrRS-IJCAI11}
Ortiz, M.; Rudolph, S.; and {\v{S}}imkus, M. 2011.
\newblock Query Answering in the Horn Fragments of the Description Logics
  {$\mathcal{SHOIQ}$} and {$\mathcal{SROIQ}$}.
\newblock In Walsh, T., ed., \emph{Proc.\ of the 22nd Int.\ Joint Conf.\ on
  Artificial Intelligence (IJCAI'11)}, 1039--1044. AAAI Press.

\bibitem[{Thi{\'e}baux, Hoffmann, and Nebel(2005)}]{ThHN-AI05}
Thi{\'e}baux, S.; Hoffmann, J.; and Nebel, B. 2005.
\newblock In Defense of {PDDL} Axioms.
\newblock \emph{Artificial Intelligence}, 168: 38--69.

\bibitem[{Tseitin(1983)}]{T-CL83}
Tseitin, G.~S. 1983.
\newblock On the Complexity of Derivation in Propositional Calculus.
\newblock In Siekmann, J.~H.; and Wrightson, G., eds., \emph{Automation of
  Reasoning: 2: Classical Papers on Computational Logic 1967--1970}, 466--483.
  Springer Berlin Heidelberg.
\newblock ISBN 978-3-642-81955-1.

\end{thebibliography}

\ifdefined\notrflagdefined
\else
\clearpage
\appendix
\onecolumn
\section{Proof of Theorem~\ref{thm:polynomial-compilation}}
\label{sec:app-rewriting}

We first describe the logic Horn-\ALCHOIQ in more detail.
We use classical DL terms here, \ie unary predicates are called \emph{concept names}, binary predicates are \emph{role names}, objects/constants are \emph{individual names} and states are \emph{ABoxes}.
We assume that all axioms in Horn-\ALCHOIQ TBoxes are in \emph{normal form} \cite{CaDK-KR18}, \ie they have one of the following shapes, where $C,C_1,\dots,C_n,D$ are concept names, $r,s$ are role names, and $a$ is an individual name:
\begin{enumerate}[label=(\roman*),leftmargin=*]
  \item $C_1\sqcap\dots\sqcap C_n\sqsubseteq D$ \label{ax:conj}
  \item $C\sqsubseteq\exists r.D$ \label{ax:ex-r}
  \item $\exists r.C\sqsubseteq D$ \label{ax:ex-l}
  \item $C\sqsubseteq {\le 1}r.D$ \label{ax:func}
  \item $C\sqsubseteq\{a\}$ \label{ax:nom-r}
  \item $r\sqsubseteq s$ \label{ax:rh}
  \item $\{a\}\sqsubseteq C$ \label{ax:nom-l}
\end{enumerate}
We added the last axiom type here since our goal is a TBox rewriting that is independent of the ABox, \ie we need to distinguish the TBox axiom $\{a\}\sqsubseteq C$ from the (equivalent) ABox fact $C(a)$.
A first-order interpretation~\Imc satisfies these axioms if it satisfies the sentences
\begin{enumerate}[label=(\roman*),leftmargin=*]
  \item $\forall x. C_1(x)\land\dots\land C_n(x)\to D(x)$,
  \item $\forall x. C(x)\to\exists y. r(x,y) \land D(y)$,
  \item $\forall x,y. r(x,y)\land C(y)\to D(x)$,
  \item $\forall x,y,z. C(x)\land r(x,y)\land D(y)\land r(x,z)\land D(z)\to y=z$,
  \item $\forall x. C(x)\to x=a$,
  \item $\forall x,y. r(x,y)\to s(x,y)$, or
  \item $C(a)$,
\end{enumerate}
respectively.
Additionally, there is a bijective and irreflexive function $\cdot^-$ on the set of role names such that $r^{--}=r$ and $\forall x,y.r(x,y)\leftrightarrow (r^-)(y,x)$ is required to hold in all models of a TBox, for all role names~$r$ ($r^-$ is called the \emph{inverse} of~$r$).

Following \citeauthor{OrRS-KR10} \shortcite{OrRS-KR10}, we will use \emph{\dsneg} rules to encode reasoning in Horn-\ALCHOIQ.
\dsneg extends \dneg rules by introducing fixed \emph{set sorts} $2^C$, where $C$ is a set of constants. Set terms of sort $2^C$ are built inductively from the constructors $\{c_1,\dots,c_n\}$ and $t_1\cup t_2$, where $c_1,\dots,c_n\in C$ and $t_1,t_2$ are set terms of this sort.
Every predicate has an associated \emph{sort function} that assigns to each position a unique sort, \ie either the (normal) \emph{element sort}, or one of the fixed set sorts.
This means that the positions of this predicates always accept only terms of the associated sort.
The semantics of the resulting (stratified) \dsneg rules is intuitive \cite{OrRS-KR10}.
%

To make it easier to compare with the existing constructions \cite{OrRS-KR10,CaDK-KR18}, in the following we write \dsneg rules with $\to$ instead of $\gets$.
%
We also use rules with conjunctions of atoms in the head (instead of only one atom).

\subsection{Encoding Instance Queries}

In the following, let \Tmc be a Horn-\ALCHOIQ TBox in normal form, and \NC/\NR/\NI denote the sets of concept names/role names/individual names in \Tmc.
In addition to the individuals in~\NI, the construction by \citeauthor{CaDK-KR18} \shortcite{CaDK-KR18} uses artificial individuals of the form~$t_X$, where $X$ is a set of concept names, and new predicates $R$ that represent a set of role names that have to be satisfied at the same time. We represent such~$R$ and~$X$ using sets of the sorts $2^{\NR}$ and $2^{\NC\cup\NI}$, respectively. The latter includes named individuals since we want to treat named and anonymous individuals using the same predicates. However, named individuals~$a$ will only appear as singleton sets~$\{a\}$.
Formally, we are not allowed to treat individuals from the ABox in this way, because then the set sorts (which need to be fixed) would depend on the ABox.
We nevertheless do this in the following and at the end of this section describe how to translate these \dsneg rules into \dneg rules that are ABox-independent.

Instead of an atom $C(x)$ \cite{CaDK-KR18}, we now use atoms of the form $\concept(C,X)$, where $C\in\NC$ is treated as a new constant and $X$ is a set as described above (\ie either a set of concept names or a singleton set containing an individual name). Likewise, role atoms $r(x,y)$ and $R(x,y)$ are transformed into $\role(r,X,Y)$ and $\roles(R,X,Y)$, respectively.
Additionally, we distinguish sets $X\subseteq\NC$ from sets $\{a\}$ with $a\in\NI$ by the predicate \anon, which is populated by the following rules, for all $C\in\NC$:
\begin{align}
&\pto \anon(\{C\})\\
\anon(X) &\to \anon(X\cup\{C\})
\end{align}
We also compute the set of inverse roles of a set~$R$, for all $r\in\NR$ \cite{OrRS-KR10}:
\begin{align}
  &\pto \inv(\{r\},\{r^-\})\\
  \inv(R,R^-) &\to \inv(R\cup\{r\},R^-\cup\{r^-\})
\end{align}
Additionally, the original rewriting uses atoms of the form $\n(x)$ and $x\approx y$, which we simply adapt to $\n(X)$ and $X\approx Y$, where $X,Y$ are as described above.
We also add a predicate \ind, which identifies all individual names from the ABox in order to identify query answers in the end.
This predicate represents a subset of the predicate \n, which will also contain anonymous individuals~$X$ that are inferred to represent a unique element in every interpretation.

The original Datalog rewriting starts with several auxiliary rules, which we translate below into their modified \dsneg form, for all $C\in\NC$ and $r\in\NR$ \cite[Figure~2]{CaDK-KR18}:
\begin{align}
  \concept(C,X) &\to \concept(\top,X) \\
  \role(r,X,Y) &\to \concept(\top,X) \land \concept(\top,Y) \displaybreak[0] \\
  \roles(R,X,Y)\land\n(Y)\land\inv(R,R^-) &\to \roles(R^-,Y,X) \\
  \roles(R,X,Y)\land r\in R &\to \role(r,X,Y) \displaybreak[0] \\
  C(x) &\to \concept(C,\{x\})\land\ind(\{x\}) \label{rule:ca} \\
  r(x,y) &\to \role(r,\{x\},\{y\})\land\ind(\{x\})\land\ind(\{y\}) \label{rule:ra} \\
  \ind(X) &\to \n(X) \displaybreak[0] \\
  X\approx Y &\to Y\approx X \\
  X\approx Y \land Y\approx Z &\to X\approx Z \displaybreak[0] \\
  \concept(C,X) \land X\approx Y &\to \concept(C,Y) \\
  \n(X) \land X\approx Y &\to \n(Y) \\
  \roles(R,X,Y)\land X\approx Z &\to \roles(R,Z,Y) \\
  \roles(R,X,Y)\land Y\approx Z &\to \roles(R,X,Z)
\end{align}
Rules~\eqref{rule:ca} and~\eqref{rule:ra} were further modified from their originals to convert the ABox predicates into our \dsneg syntax in an ABox-independent way. In a slight abuse of notation, on the left-hand side of Rule~\eqref{rule:ca}, $C$ is treated as a concept name, and on the right-hand side $C$ is treated as a constant (similarly for Rule~\eqref{rule:ra}).

The axioms of the TBox~\Tmc are translated into the following \dsneg rules \cite[Figures~3 and~4]{CaDK-KR18}.
For every axiom of the form~\ref{ax:conj}:
\begin{align}
  \concept(C_1,X)\land\dots\land\concept(C_n,X)&\to\concept(D,X)
\intertext{For axioms of the form~\ref{ax:ex-r}:}
  \concept(C,X)&\to\role(r,X,\{D\}) \label{rule:c-exists-r-d}
\intertext{For axiom type~\ref{ax:ex-l}:}
  \role(r,X,Y)\land\concept(C,Y) &\to \concept(D,X) \\
  \concept(C,X)\land \roles(R^-,X,Y)\land
  r\in R\land\inv(R,R^-)\land\anon(Y) &\to \roles(R^-,X,Y\cup\{D\}) \label{rule:exists-inv}
\intertext{Axiom type~\ref{ax:func} requires more complex rules:}
  \concept(D,Y)\land\role(r^-,Y,X)\land
    {} & \notag \\
    \concept(C,X)\land\role(r,X,Z)\land\concept(D,Z)\land\n(Z) &\to Y\approx Z \displaybreak[0] \\
  \concept(C,X)\land r\in R\land\roles(R,X,Y)\land
    \concept(D,Y)\land{} & \notag \\
    \anon(Y)\land r\in S\land\roles(S,X,Z)\land\concept(D,Z)\land\anon(Z) &\to \roles(R\cup S,X,Y\cup Z) \displaybreak[0] \label{rule:func-merge} \\
  \concept(C,X)\land r\in R\land\inv(R,R^-)\land
    \concept(D,Y)\land \roles(R^-,Y,X)\land{} & \notag \\
    r\in S\land\inv(S,S^-)\land\anon(Z)\land E\in Z\land
    \roles(S,X,Z)\land\concept(D,Z) &\to \concept(E,Y) \land{} \notag \\
    &\pto \roles(R^-\cup S^-,Y,X) \displaybreak[0] \\
  \concept(D,Y)\land\role(r^-,Y,X)\land
    \concept(C,X)\land\n(X) &\to \n(Y)
\intertext{In addition, we need the following rule that completes the translation of axioms of types \ref{ax:ex-r}--\ref{ax:func} by adding the newly created anonymous individuals to the required concepts:}
  \role(r,X,Y)\land C\in Y &\to \concept(C,Y)
\intertext{For axiom type~\ref{ax:nom-r}:}
  \concept(C,X) &\to \{a\}\approx X \land \n(\{a\})
\intertext{For axiom type~\ref{ax:rh}, we use the predicate \sup that connects each role name to the set of all its super-roles (Rule~\ref{rule:rh} is instantiated for every $s\sqsubseteq t\in\Tmc$ or $s^-\sqsubseteq t^-\in\Tmc$):}
  &\to \sup(r,\{r\}) \\
  \sup(r,S)\land s\in S &\to \sup(r,S\cup\{t\}) \label{rule:rh} \\
  \role(r,X,Y)\land\sup(r,S) &\to \roles(S,X,Y) \\
  \role(r^-,X,Y)\land\sup(r,S)\land\inv(S,S^-) &\to \roles(S^-,X,Y)
\intertext{Finally, axiom type~\ref{ax:nom-l} can be handled like a concept assertion:}
  &\to \concept(C,\{a\})\land\n(\{a\})
\end{align}

So far, the rules did not use negation and can be seen as the first stratum of the final rule set.
This part is already a rewriting of any instance query over the TBox signature, \ie $\concept(C,\{a\})$ is contained in the least Herbrand model of these rules and an ABox iff $C(a)$ is entailed by the original TBox over the ABox \cite[Lemma~4]{CaDK-KR18}.
To be able to answer arbitrary UCQs, we also need to encode the so-called \emph{filtration phase} \cite{CaDK-KR18}.

\subsection{Building a Canonical Model}

The next stratum encodes Definition~7 from \citeauthor{CaDK-KR18} \shortcite{CaDK-KR18} by using negated atoms over the predicate~\n.
The goal of this part is to derive a dependency relation $\role_\tr(r,X,Y)$ that encodes the order in which individuals are created during the construction of a model of the ontology.
Extended individuals of the form $t_{R,X}^i$ are used in this relation, where $R\subseteq\NR$, $X\subseteq\NC$ and $i\in\{0,1,2\}$ \cite{CaDK-KR18}.
Therefore, the sets $X,Y$ in $\role_\tr(r,X,Y)$ are now considered to be subsets of $\NR\cup\NC\cup\NI\cup\{0,1,2\}$, where again individuals can only occur in singleton sets~$\{a\}$, and at most one index~$0,1,2$ can be present in any given set.
In addition to $\role_\tr$, the following rules also compute extensions $\role'$ and $\concept'$ of $\role$ and $\concept$, respectively, to the new individuals.
Together, these three predicates describe a kind of \emph{canonical model} (called~$\Cmc_\Omc$) over which UCQs will be answered.

As a prerequisite, we need to introduce an intermediate stratum to define a total order~$\preceq$ on sets of the form $R\cup X$.
For this purpose, we consider an enumeration $r_1,\dots,r_n,C_{n+1},\dots,C_{n+m}$ of all role and concept names and use the following rules to define the lexicographic order~$\preceq$ based on the auxiliary relations~$\prec_i$ and~$\approx_i$, $i\in\{1,\dots,n+m\}$ (all using infix notation):
\begin{align}
  &\pto \emptyset\approx_1\emptyset \\
  &\pto \emptyset\prec_1\{r_1\} \\
  &\pto \{r_1\}\approx_1\{r_1\} \displaybreak[0] \\
  X\prec_1 Y &\to X\prec_2 Y \\
  X\approx_1 Y &\to X\approx_2 Y \displaybreak[0] \\
  X\approx_1 Y &\to X\prec_2 Y\cup\{r_2\} \\
  X\approx_1 Y &\to X\cup\{r_2\}\approx_2 Y\cup\{r_2\} \\
  &\ \ \vdots \notag \\
  X\approx_{n+m}Y &\to X\preceq Y \\
  X\prec_{n+m}Y &\to X\preceq Y
\end{align}

The following rules, which use the negations of~$\preceq$ and~\n from the previous strata, correspond to Definition~7 from \citeauthor{CaDK-KR18} \shortcite{CaDK-KR18}, where $j=(i+1)\mod 3$:%
\begin{align}
  \n(X)\land\role(r,X,Y)\land\n(Y) &\to \role_\tr(r,X,Y)\land\role_\tr(r^-,Y,X) \\
  \n(X)\land\roles(R,X,Y)\land
    \anon(Y)\land\lnot\n(Y)\land r\in R &\to\role_\tr(r,X,R\cup Y\cup\{0\}) \displaybreak[0] \label{rule:co-init} \\
  \role_\tr(r,X,R\cup Y\cup\{i\})\land
    \roles(S,Y,Z)\land\anon(Z)\land\lnot\n(Z)\land{} & \notag \\
    s\in S\land R\cup Y\preceq S\cup Z &\to \role_\tr(s,R\cup Y\cup\{i\},S\cup Z\cup\{j\}) \\
  \role_\tr(r,X,R\cup Y\cup\{i\})\land
    \roles(S,Y,Z)\land\anon(Z)\land\lnot\n(Z)\land{} & \notag \\
    s\in S\land R\cup Y\not\preceq S\cup Z &\to \role_\tr(s,R\cup Y\cup\{i\},S\cup Z\cup\{i\}) \\
  \role_\tr(s,X,R\cup Y\cup\{i\})\land
    \role(r,Y,Z)\land\n(Z) &\to \role_\tr(r,R\cup Y\cup\{i\},Z)\land{} \\
    &\pto \role_\tr(r^-,Z,R\cup Y\cup\{i\}) \displaybreak[0] \\
  \role_\tr(s,X,Y)\land\concept(C,Y) &\to \concept'(C,Y) \\
  \role_\tr(s,X,R\cup Y\cup\{i\})\land
    \concept(C,Y) &\to \concept'(C,R\cup Y\cup\{i\}) \\
  \role_\tr(r,X,Y) &\to \role'(r,X,Y)\land\role'(r^-,Y,X)
\end{align}
The least Herbrand model of these rules (restricted to $\role_\tr$, $\concept'$, and $\role'$) corresponds to the set~$\Cmc_\Omc$ \cite{CaDK-KR18}.
Conditions of the type \enquote{$t_{R,Y}^i$ is in~$\Cmc_\Omc$} are translated into body atoms like $\role_\tr(r,X,R\cup Y\cup\{i\})$ since these new individuals are only introduced into~$\Cmc_\Omc$ inside of $\role_\tr$-facts.

\subsection{Encoding the Filtration}

Finally, we encode Definition~8 from \citeauthor{CaDK-KR18} \shortcite{CaDK-KR18}, which constructs a family of graphs that use the variables of a given CQ as vertices, and their edges encode possible matches of the CQ into~$\Cmc_\Omc$.
Some of these matches have to be filtered out since they lead to spurious answers.
We assume that the input CQ~$q$ is of the form $\exists v_1,\dots,v_\ell.\phi(v_1,\dots,v_k)$, \ie $v_{\ell+1},\dots,v_k$ are the free variables (UCQs can be treated by encoding each component CQ individually and then merging the results in a single predicate).
By $\phi(V_1,\dots,V_k)$ we denote the result of replacing each concept atom $C(v_i)$ in~$\phi$ by~$\concept'(C,V_i)$ and similarly $r(v_i,v_j)$ by $\role'(r,V_i,V_j)$, where each~$V_i$ is viewed as a set variable over $\NR\cup\NC\cup\{i\}$ and denotes a possible mapping of~$v_i$ into~$\Cmc_\Omc$.
We use the indices $1,\dots,k$ to refer to the vertices in the graphs that are constructed, and atoms of the form $\edge(i,j,V_1,\dots,V_k)$ to denote an edge from~$i$ to~$j$ (which depends on a specific instantiation of all variables by individuals in~$\Cmc_\Omc$).

The following are the rules corresponding to the first part of Definition~8 from \citeauthor{CaDK-KR18} \shortcite{CaDK-KR18}, for all role atoms $r(v_i,v_j)$ in~$\phi$:
\begin{align}
  \phi(V_1,\dots,V_k)\land\role_\tr(r,V_i,V_j)\land\lnot\role_\tr(r^-,V_j,V_i) &\to \edge(i,j,V_1,\dots,V_k) \\
  \phi(V_1,\dots,V_k)\land\role_\tr(r^-,V_j,V_i)\land\lnot\role_\tr(r,V_i,V_j) &\to \edge(j,i,V_1,\dots,V_k)
\intertext{Now, for all $i,j\in\{1,\dots,k\}$, we apply the following rules to collapse this graph according to Definition~8 from \citeauthor{CaDK-KR18} \shortcite{CaDK-KR18}:}
  \edge(i,j,V_1,\dots,V_k)\land\edge(m,j,V_1,\dots,V_k) &\to \equal(i,m,V_1,\dots,V_k) \\
  \edge(i,j,V_1,\dots,V_k)\land\edge(m,n,V_1,\dots,V_k)\land
    \equal(j,n,V_1,\dots,V_k) &\to \equal(i,m,V_1,\dots,V_k)
\intertext{We now need to check whether the resulting graph is a rooted directed forest, \ie contains no cycles and no \enquote{diamonds} that reconnect different branches:}
  \edge(i,j,V_1,\dots,V_k) &\to \reach(i,j,V_1,\dots,V_k) \\
  \reach(i,j,V_1,\dots,V_k) \land \equal(j,m,V_1,\dots,V_k) &\to \reach(i,m,V_1,\dots,V_k) \\
  \reach(i,j,V_1,\dots,V_k) \land \equal(i,m,V_1,\dots,V_k) &\to \reach(m,j,V_1,\dots,V_k) \displaybreak[0] \\
  \reach(i,j,V_1,\dots,V_k) \land \edge(j,m,V_1,\dots,V_k) &\to \reach(i,m,V_1,\dots,V_k) \\
  \reach(i,i,V_1,\dots,V_k) &\to \bad(V_1,\dots,V_k) \\
  \edge(i,j,V_1,\dots,V_k)\land\edge(i,m,V_1,\dots,V_k)\land
    \lnot\equal(j,m,V_1,\dots,V_k)\land{} & \notag \\
    \reach(j,n,V_1,\dots,V_k)\land\reach(m,n,V_1,\dots,V_k) &\to \bad(V_1,\dots,V_k)
\intertext{Finally, we can use the predicate \bad to filter out spurious matches and return the actual answers to~$q$ in the predicate~$P_q$:}
  \phi(V_1,\dots,V_k)\land\lnot\bad(V_1,\dots,V_k) &\to P_q'(V_{\ell+1},\dots,V_k) \\
  P_q'(V_{\ell+1},\dots,V_k)\land
    \ind(V_{\ell+1})\land\dots\land\ind(V_k)\land
    a_{\ell+1}\in V_{\ell+1}\land\dots\land a_k\in V_k &\to P_q(a_{\ell+1},\dots,a_k)
\end{align}

\subsection{From \dsneg to \dneg}

To simulate set terms in plain \dneg, we adapt an existing construction, which did not deal with negation or ABox-independent rule sets \cite{OrRS-KR10}.
First, each set union~$t_1\cup t_2$ over the domain~$S$ is replaced with a fresh variable~$X$ and the ternary atom~$\U_S(t_1,t_2,X)$ is added to the body of the rule in which this term occurs.
New rules are added to simulate the set union with this predicate (see below).
Then, sets~$X\subseteq S$ are represented as bit vectors of length~$|S|$ and set variables as vectors of variables, and all atoms are replaced accordingly.
This gets rid of all set expressions while increasing the arity of the predicates polynomially.

The main problem we face here is that we used singleton sets $\{a\}$ to refer to individual names~$a$ from the ABox, which means that the set sort $2^{\NC\cup\NI}$ was treated as if it contained all these individual names, although our translation cannot depend on the ABox (see Definition~\ref{def:datalog-rewriting}).
To avoid this issue, we modify the bit vector encoding above to directly represent constants by themselves.
For this, recall that individual names~$a$ can only occur in singleton sets~$\{a\}$, because sets~$X$ are only extended if they belong to~\anon (see, \eg Rules~\eqref{rule:exists-inv}, \eqref{rule:func-merge}, or \eqref{rule:co-init}).
Thus, we can represent each instantiated set~$X$, which is either of the form~$\{a\}$ or a subset of~\NC, as a vector $(a,0,\dots,0)$ or $(0,b_1,\dots,b_m)$, respectively, where $m=|\NC|$ and $b_i=1$ iff the $i$-th concept name of~\NC is in~$X$ (according to some fixed enumeration of~\NC).
The new constants~$0$ and~$1$ are used to represent bit values.
Correspondingly, set variables~$X$ are split into vectors $(x_0,x_1,\dots,x_m)$, where $x_0$ holds the individual name (if any) and $x_1,\dots,x_m$ represent a subset of~\NC.
The encoding works similarly for sets of the sorts $2^\NR$, $2^{\NR\cup\NC}$, and $2^{\NR\cup\NC\cup\NI\cup\{0,1,2\}}$ that are employed in the rewriting above.
For example, the \dneg versions of Rules~\eqref{rule:ca} and~\eqref{rule:c-exists-r-d} are
\begin{align}
  C(x) &\to \concept(C,x,0,\dots,0) \land \ind(x,0,\dots,0), \text{ and} \\
  \concept(C,x_0,\dots,x_m) &\to \role(r,x_0,\dots,x_m,0,b_1,\dots,b_m),
\end{align}
respectively, where $b_i=1$ iff $D$ is the $i$-th concept name in~\NC.

Apart from the new predicates like~$\U_{\NC\cup\NI}$, this encoding clearly preserves the stratification of the original rule set.
The following additional rules are used to define~$\U_{\NC\cup\NI}$, and similarly for the other set sorts:
\begin{align}
  &\to \max(0,0,0) \\
  &\to \max(1,0,1) \\
  &\to \max(0,1,1) \\
  &\to \max(1,1,1) \\
  \max(x_1,y_1,z_1)\land\dots\land\max(x_m,y_m,z_m) &\to \U_{\NC\cup\NI}(0,x_1,\dots,x_m,0,y_1,\dots,y_m,0,z_1,\dots,z_m),
\end{align}
This suffices to define the set union since we never need to compute unions involving singleton sets~$\{a\}$ with $a\in\NI$.
These additional rules can be included in the first stratum since \max and~$\U_S$ do not occur in any other rule heads.

This finishes the presentation of the \dneg rewriting for any UCQ over a Horn-\ALCHOIQ TBox.
Its correctness follows mainly from an existing result \cite[Theorem~3]{CaDK-KR18} since we only translated the relevant definitions into \dsneg rules.
It can also be verified that the resulting set of \dneg rules is of polynomial size.

\section{Proof of Theorem~\ref{thm:horn-sroiq-no-polynomial-compilation}}
\label{sec:app-horn-sroiq-no-polynomial-compilation}

We show how to construct the eKAB task $(\Delta_n,\Omc,I_w,g)$ such that $M$ accepts a word~$w$ of length~$n$ iff there is a plan of length~$1$.
The construction is based on the \TwoExpTime-hardness proof for Horn-\SROIQ \cite{OrRS-KR10} and uses only a single action with precondition $[\exists x.B(x)]$ and unconditional effect~$g$.
We do not repeat all details of the original construction here, but only adapt the relevant parts.
The proof encodes the Turing machine~$M$ and an input word~$w$ into a TBox~\Tmc using the two objects~$o$ and~$e$ such that~$M$ accepts~$w$ iff at least one unary predicate $H_{q_f}$ (representing a final state of the TM) is empty in every model of~\Tmc.
The original proof then goes on to add axioms $H_{q_f}\sqsubseteq\bot$ to force \Tmc to become unsatisfiable in such a case.
Since we require the initial state to be consistent with the TBox, we instead use the axioms $H_{q_f}\sqsubseteq B$, which allows us to query for $\exists x.B(x)$ instead of checking unsatisfiability.

We further adapt the original reduction by extracting from~\Tmc the description of the input word~$w$ into a state~$I_w$.
For $w=w_0\dots w_{n-1}$, \Tmc contains the following axioms to encode the input tape (notation is slightly adjusted to avoid clashes):
\begin{align}
  \{o\} & \sqsubseteq I_1\sqcap H_{q_0} \label{ax:tm-zero} \\
  I_j &\sqsubseteq A_{w_j}\sqcap\forall h.I_{j+1} & (0\le j<n) \label{ax:tm-j} \displaybreak[0] \\
  I_j &\sqsubseteq \overline{H_r} & (1\le j<n) \label{ax:tm-no-head} \displaybreak[0] \\
  I_{n} &\sqsubseteq A_{\Box} \label{ax:tm-blank} \\
  I_{n} &\sqsubseteq \forall h.I_{n} \label{ax:tm-right}
\end{align}
The object~$o$ indicates the starting point of the tape, and the unary predicates~$I_j$ identify the first~$n$ cells, which are connected via the binary predicate~$h$. The predicates~$A_{w_j}$ indicate the presence of the input symbols in these cells. The predicates~$H_{q_0}$ and~$\overline{H_r}$ indicate the presence and absence of the head, respectively.
Finally, all cells to the right of the input word are labeled with a blank symbol~$\Box$ by using the auxiliary predicate~$I_{n}$.
We leave the axioms~\eqref{ax:tm-no-head}--\eqref{ax:tm-right} in the TBox, but replace \eqref{ax:tm-zero}--\eqref{ax:tm-j} by the following axioms \eqref{ax:thm-j-new}--\eqref{ax:tm-right-new} and assertions for~$I_w$ \eqref{ax:tm-h-zero}--\eqref{ax:tm-j-new-}:
\begin{align}
  S_{j,a} &\sqsubseteq \forall h^j.I_j\sqcap A_a & (0\le j<n,\ a\in\Sigma) \label{ax:thm-j-new} \\
  I_{n-1} &\sqsubseteq \forall h_n.I_{n} \label{ax:tm-right-new} \displaybreak[0] \\
  & H_{q_0}(o) \label{ax:tm-h-zero} \\
  & S_{j,w_j}(o) & (0\le j<n) \label{ax:tm-j-new-}
\end{align}
Here, $\forall h^j$ stands for $j$ nested restrictions of the form $\forall h$, and $S_{j,a}$ indicates the presence of the symbol~$a$ of the TM alphabet~$\Sigma$ at cell~$j$.
In this way, the final TBox~$\Tmc_n$ only depends on the length~$n$ of the input word~$w$, but not on~$w$ itself. The final domain description is $\Delta_n=(\Pmc_n,\Amc_n,\Tmc_n)$, where $\Pmc_n$ contains all symbols from~$\Tmc_n$ as well as~$g$, and $\Amc_n$ consists of the single action described above.

\section{Proof of Theorem~\ref{thm:sh-alci-no-polynomial-compilation}}
\label{sec:app-sh-alci-sroiq-no-polynomial-compilation}

The proof follows the same arguments as for Theorem~\ref{thm:horn-sroiq-no-polynomial-compilation}, the only difference being how the domain description~$\Delta_n$ and state~$I_w$ are obtained.
For CQ entailment in \ALCI, we adapt a reduction from a (universal) \AExpSpace Turing machine \cite{Lutz-DL07}.
The single action we use in the reduction will have a precondition $[q_w]$, where~$q_w$ is the CQ from that reduction, which, despite its name, does not depend on the input word $w=w_0\dots w_{n-1}$, but only on the length~$n$.
Again, the only adaption we have to do is to extract the part of the TBox encoding the input word into a state~$I_w$.
Consider the relevant axioms (again with slightly adapted notation), where $0\le j<n$ \cite{Lutz-DL07}:
\begin{align}
  & \big(R\sqcap I\big)(o) \label{ax:r-and-i} \\
  I &\sqsubseteq \forall s^{n+2}.\big((G_h\sqcap(\pos=j))\to w_j\big) \label{ax:i-gh-wj} \displaybreak[0] \\
  I &\sqsubseteq \forall s^{n+2}.\big((G_h\sqcap(\pos=0))\to q_0\big) \\
  I &\sqsubseteq \forall s^{n+2}.\big((G_h\sqcap(\pos\ge n))\to b\big)
\end{align}
Here, $R\sqcap I$ describes the starting point of the initial configuration and its $s^{n+2}$-successors marked with~$G_h$ identify the exponentially many tape cells. The counter~\pos identifies particular cells, $w_j$ describes the tape content, $q_0$ the initial state, and $b$ the blank symbol.
We use a similar trick as before to split~\eqref{ax:i-gh-wj} into ABox facts and TBox axioms that do not depend on the input word~$w$, but only on its length (for all $0\le j<n$, $a\in\Sigma$):
\begin{align}
  S_{j,a} &\sqsubseteq\forall s^{n+2}.\big((G_h\sqcap(\pos=j))\to a\big) \\
  & S_{j,w_j}(o) \label{ax:wj}
\end{align}
The state~$I_w$ now consists of all facts~\eqref{ax:wj} as well as~\eqref{ax:r-and-i}, and all other axioms are part of~$\Tmc_n$.
The remaining arguments are the same as in the proof of Theorem~\ref{thm:horn-sroiq-no-polynomial-compilation}.

The reduction for CQ entailment over \SH TBoxes \cite{ELOS-INFSYS09} is very similar to the previous case, except that the predicates~$R$ and~$G_h$ are not used, $s^{n+2}$ is replaced by $r^{n+1}$, $w_j$ is replaced by $\forall r.(E_h\to w_j)$, and similarly for~$q_0$ and~$b$ (many other details not relevant here are different as well).
Hence, we can use very similar adaptations.

\section{Benchmark Description}
\label{sec:app-benchmarks}
Our collection of benchmarks consists of a total of $235$ instances adapted from the publicly available
DL-Lite eKAB benchmark collection~\cite{BHKS-KR21} as well as newly developed high expressivity domains. The benchmarks and the compiler are available in the supplementary material. Each problem instance has two representations: the Horn-\SHIQ eKAB task encoding with an ontology written in Turtle\footnote{\url{https://www.w3.org/TR/turtle/}} and its compilation into PDDL.
 

\paragraph{Adapted DL-Lite eKAB benchmarks:}
We translated the original benchmarks into equivalent representations in our Horn-\SHIQ eKAB task encoding.
Detailed descriptions of these domains are available online. 
In short,
\begin{itemize}
\item 
in {\robot}~\cite{CMPS-IJCAI16}, a robot is positioned on a grid without knowing its position and the goal is to reach a target cell. The ontology describes relations between rows and columns.

\item The goal of {\taskassignment}, inspired by~\cite{CMPS-IJCAI16}, is to hire two electronic engineers for a company, while the ontology describes relations between different job positions. 

\item The \elevator and \cats benchmarks are inspired by standard planning benchmarks. In the {\cats} domain, there is a set of packages that contain either cats or bombs and the task is to disarm all bombs. An elevator in the {\elevator} benchmark can move up and down between floors to serve passengers according to their origins and destinations.

\item Both the {\vta} and {\tpsa} benchmarks are adaptations from older work on semantic web-service composition~\cite{hoffmann:etal:icwe-08}. {\vtaroles} is a more complex variant of \vta.
\end{itemize}

\paragraph{High Expressivity Domains:} 

\begin{itemize}
\item \drones models a complex 2D drone navigation problem, in which drones need to be
moved while avoiding certain situations; the latter is given by ontology
reasoning, involving Horn concept inclusions with qualified existential
restrictions occurring negatively and symmetric roles. Grid cells are occupied with different objects like $\textsf{Human}$s or $\textsf{Tree}$s or weather conditions like $\textsf{LowVisibility}$ or $\textsf{Rain}$. There is a set of $\textsf{Drone}$s randomly placed on the board. Depending on the distances to other objects, a $\textsf{Drone}$ can enter a critical state (defined by the ontology).
The goal is to move the drones such that no two drones in a critical state are next to each other. In the benchmark, instances vary in the board size and the number of drones.
We have chosen the instances such that some of them remain hard for the planner to solve. The compilation itself is always very fast.
\item  \queens generalizes the eight queens
puzzle from chess to variable numbers of board sizes, $n\in\{5,\dots, 10\}$, and queens,
$m\in\{n-4,\dots, n\}$. In the initial state, queens are placed randomly and the ontology contains a symmetric, transitive role to describe which queen movements are legal.  Similarly to \drones, the planner must find a sequence of legal moves such that no two queens threaten each other.
\item   \robotconj is a redesign of \robot,
moving complexity from action descriptions into the ontology. The original
\robot benchmark encoded static knowledge about 2D-grid cell adjacency in the
action descriptions, which via the use of Horn clauses can be encoded much more
naturally directly in the ontology.
More precisely, in a slightly simplified notation, the action \textsf{MoveDown} contains the two redundant conditional effects
\begin{align}
\textsf{	(when (and (AboveOf1 ?x) (BelowOf2 ?x))\,} \label{eq:robot:non-strict}\\
\textsf{	(Row0 ?x))},\nonumber
\end{align}
which one can read as \enquote{if the robot is above or in row 1 and below row 2, then move the robot to row 0}, and
\begin{align}
\textsf{	(when (Row1 ?x))\,} \label{eq:robot:strict} \\
\textsf{	(Row0 ?x))},\nonumber
\end{align}
\ie \enquote{if the robot position is in row 1, then move the robot to row 0}. However, encoding the static knowledge that \textsf{AboveOf1} and \textsf{BelowOf2} imply \textsf{Row1} is beyond DL-Lite.
Moreover, the actions \textsf{MoveUp}, \textsf{MoveLeft}, and \textsf{MoveRight} have similar redundant effects.
%
For \robotconj we have simplified these descriptions by using axioms like 
$
\textsf{AboveOf1} \sqcap \textsf{BelowOf2} \sqsubseteq \textsf{Row1}
$, which allows us to get rid of~\eqref{eq:robot:non-strict}. 
\end{itemize}

\fi

\end{document}